\documentclass{article}

 \usepackage[main, final]{neurips_2026}

\usepackage[utf8]{inputenc} %
\usepackage[T1]{fontenc}    %
\usepackage{hyperref}       %
\usepackage{url}            %
\usepackage{booktabs}  
\usepackage{multirow}%
\usepackage{amsfonts}       %
\usepackage{nicefrac}       %
\usepackage{microtype}      %
\usepackage{xcolor}         %
\usepackage{xspace}
\usepackage{enumitem}
\usepackage{graphicx}
\usepackage{algorithm}
\usepackage{algpseudocode}
\usepackage{subcaption}
\usepackage{placeins}
\usepackage{wrapfig}
\usepackage{makecell}

\usepackage{amsmath,amsfonts,bm}

\def\eqref#1{equation~\ref{#1}}

\def\1{\bm{1}}

\def\vtheta{{\bm{\theta}}}

\def\mM{{\bm{M}}}

\def\mS{{\bm{S}}}
\def\mT{{\bm{T}}}

\DeclareMathAlphabet{\mathsfit}{\encodingdefault}{\sfdefault}{m}{sl}
\SetMathAlphabet{\mathsfit}{bold}{\encodingdefault}{\sfdefault}{bx}{n}

\DeclareMathOperator*{\argmax}{arg\,max}

\usepackage{amsthm}
\theoremstyle{plain}

\newtheorem{proposition}{Proposition}

\theoremstyle{definition}
\newtheorem{definition}{Definition}
\newtheorem{problem}{Problem}

\theoremstyle{remark}
\newtheorem{remark}{Remark}

\newif\ifcomments
\commentstrue

\usepackage{pifont}
\usepackage{amssymb}
\usepackage{comment}

\everypar{\looseness=-1}

\usepackage[dvipsnames]{xcolor}
\usepackage{hyperref}
\hypersetup{
     colorlinks = true,
     linkcolor = BrickRed,
     anchorcolor = blue,
     citecolor = Blue,
     filecolor = blue,
     urlcolor = black
}

\title{Understanding Layer Patching in\\Model Size Interpolation}

\author{%
Sara Kangaslahti\textsuperscript{1},\;
Jonathan Geuter\textsuperscript{1,2}\thanks{Equal contribution.},\;
Nihal V.~Nayak\textsuperscript{1}\footnotemark[1],\;
Marco Fumero\textsuperscript{3}\footnotemark[1],\\
\textbf{Francesco Locatello}\textsuperscript{3},\;
\textbf{David Alvarez-Melis}\textsuperscript{1,2}\\
\textsuperscript{1}Harvard University \quad
\textsuperscript{2}Kempner Institute \quad
\textsuperscript{3}IST Austria\\
\texttt{sarakangaslahti@g.harvard.edu}
}

\begin{document}

\maketitle
\begin{abstract}

Layer patching is a key technique for model size interpolation in large language models (LLMs), in which a subset of a distilled student model’s layers is replaced with corresponding teacher-model layers to construct intermediate-sized models without additional training. 
Prior work largely ignores the order in which layers are patched, which we show can dramatically affect the performance and reliability of the resulting models. In this work, we study the role of layer patching in model size interpolation. 
We cast finding the optimal layer subset for each model size as an optimization problem and prove it can be viewed as a shortest-path problem in a certain acyclic graph.
In experiments, we show that patching strongly shapes interpolation behavior, with effects that vary substantially across model families. We find that simple sequential strategies---patching either from the first layer to the last or from the last to the first---often achieve surprisingly strong performance in practice. We further introduce KLPatch, a greedy patching algorithm based on KL divergence, which often improves over last-to-first patching and approximately solves the optimization problem.
Together, our results provide a principled understanding of how layer patching affects model size interpolation and offer practical guidance for constructing near-optimal interpolated models. The code and
models are available at https://github.com/dcml-lab/boomerang-distillation.

\looseness-1

\looseness-1
\end{abstract}

\section{Introduction}
\vspace{-0.2cm}

Creating LLM families with different model sizes is one of the most practical ways to adapt to users' varying compute constraints~\citep{huyen2022designing,khandelwal2025k,team2026kimi}.
However, pre-training LLMs of different sizes requires separate training runs, which often results in the release of only a small number of coarse-grained, fixed-size LLMs~\citep{grattafiori2024llama,yang2025qwen3,qwen35blog}.
\emph{Layer patching} offers a training-free way to fill the gap between two such models. Starting from the smaller model, individual layers are replaced by their counterparts from the larger model, yielding intermediate models whose size and performance lie between those of the two endpoints.
In boomerang distillation~\citep{kangaslahti2026boomerang,zhou2026thinking}, layers of a student LLM are patched with contiguous blocks of teacher layers once the student has been distilled from the teacher model.
A central design choice lies in \emph{which} layers to patch for a given target size, a key user-facing decision with a combinatorial design space.
As we show, this choice strongly shapes the resulting interpolation, yet it has remained poorly understood.
\looseness-1

Prior work on boomerang distillation provides limited guidance on how to patch student models for optimal interpolation performance.
\citet{kangaslahti2026boomerang} consider only two patching strategies: patching from the first layer to the last and patching from the last layer to the first.
Although they show that some LLMs prefer one order over the other, the broader design space of patching remains largely unexplored.
Moreover, in Section \ref{sec:patching-order-permutations}, we show that naively increasing the size of an interpolated model by arbitrarily patching the student can degrade performance, contradicting the conventional view that larger models perform better~\citep{sutton2019bitter,kaplan2020scaling,hoffmann2022training}.
These observations motivate a deeper understanding of patching to identify a generalizable recipe for constructing reliable interpolated models.

In this work, we provide a comprehensive understanding of layer patching for model interpolation.
We define layer patching as an optimization problem over interpolations (Section~\ref{sec:interpolation-optimization}), which applies to any pair of smaller and larger models with aligned layers, accounts for patches of different sizes, and includes interpolations that need not correspond to nested subsets of model layers, moving beyond the setup of \citet{kangaslahti2026boomerang}.
As its computational complexity makes it infeasible to solve, we introduce an alternative shortest-path problem (Section~\ref{sec:patching-orders}) which we empirically and theoretically show approximates the optimal model interpolation.
In experiments, we first exhaustively search over all possible interpolations in small language models~\citep{devlin2019bert,radford2019language} such as DistilBERT and DistilGPT to find the globally optimal model size interpolation (Section~\ref{sec:experiment:distilbert}).
Next, we extend our insights from small language models to larger models such as Qwen and Pythia~\citep{biderman2023pythiasuiteanalyzinglarge,yang2025qwen3}, and run an extensive study over 200 randomly sampled patching orders for each model (Section~\ref{sec:experiment:patching_order}).
Finally, we propose \textit{KLPatch} (Section~\ref{sec:klpatch}), a greedy KL divergence-based patching algorithm that approximately solves the shortest-path problem, reducing the computational complexity for finding performant interpolations from $O(2^{N})$ to $O(N^2)$, where $N$ is the number of patchable layers, and apply it both to boomerang distillation and to interpolation between a quantized LLM and its higher-precision counterpart, another setting in which model size interpolation has been studied~\citep{dumitru2025variable}.
\looseness-1

Our experiments reveal valuable insights into how patching order drives model interpolation.
Our exhaustive search over all possible patching orders in DistilBERT and DistilGPT2 shows that, remarkably, patching from the last student layer to the first produces the globally optimal interpolation for DistilBERT and a top-5\% interpolation for DistilGPT2.
For LLMs, the patching recipe from~\citet{kangaslahti2026boomerang} is not always optimal or near-optimal but remains competitive, and shortest-path solutions tend to perform similarly to optimal interpolations.
Finally, our experiments show that KLPatch often outperforms first-to-last and last-to-first patching and finds near-best interpolations among tested orders across several language models in boomerang distillation. Furthermore, the same algorithm yields strong interpolations between quantized and full-precision LLMs.
Overall, our findings establish patching order as a key design choice for building interpolated models.

Our work makes the following contributions:
\begin{itemize}[leftmargin=*]
\item \textbf{Formalization of layer patching.} We formalize layer patching between a smaller and a larger model as an optimization problem over interpolation curves (Section~\ref{sec:interpolation-optimization}), and derive an alternative shortest-path problem, which we empirically and theoretically justify (Sections~\ref{sec:patching-orders}, \ref{sec:patching-order-permutations}).

\item \textbf{Phenomenology of patching order.} We provide the first thorough study of layer patching. In particular, we perform exhaustive sweeps over \textit{all} interpolations in DistilBERT and DistilGPT (Section~\ref{sec:experiment:distilbert}), and 200 random patching orders on Qwen3 4B, Qwen3 8B, and Pythia 6.9B (Section~\ref{sec:experiment:patching_order}), showing that simple last-to-first patching sometimes remains competitive, while optimal permutations obtained from solving the shortest-path problem closely approximate the optimal interpolation.
\looseness-1

\item \textbf{Practical patching algorithm.} We introduce KLPatch, an efficient greedy algorithm that approximately solves the interpolation optimization problem via an interpolation graph (Section~\ref{sec:klpatch-algo}), reducing the computational complexity from $O(2^N)$ to $O(N^2)$ and empirically yielding model interpolations close to the best observed performance for boomerang distillation (Section~\ref{sec:experiment:klpatch}) and for interpolation between quantized and bf16 LLMs (Section~\ref{sec:experiment:klpatch-quantization}).
\looseness-1
\end{itemize}

\begin{figure}
    \centering
    \includegraphics[width=1\linewidth]{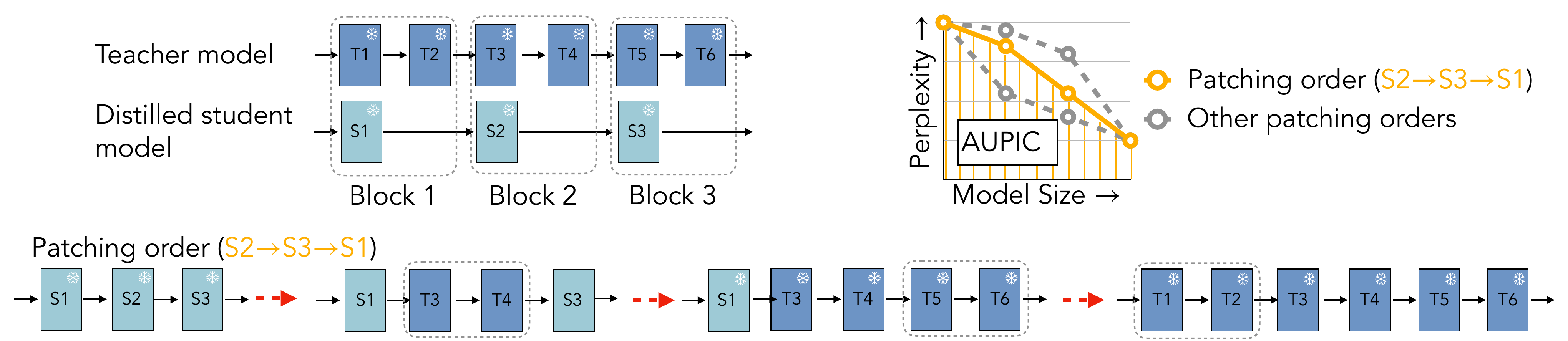}
    \caption{\textbf{Patching order in model size interpolation.}
    We show boomerang distillation with a six-layer teacher and a three-block distilled student, where each student block corresponds to a contiguous pair of two-layer teacher blocks. 
    A patching order specifies the sequence in which student layers are replaced by their corresponding teacher blocks, inducing a trajectory of intermediate-size models from the student to the teacher. 
    Here, the order $S2 \rightarrow S3 \rightarrow S1$ first replaces the middle student block, then the last block, and finally the first block. 
    Different patching orders induce different interpolation curves, which we summarize using the area under the perplexity interpolation curve (AUPIC).\looseness-1}
    \vspace{-0.45cm}
    \label{fig:main-patching-overview}
\end{figure}
\vspace{-0.2cm}
\section{Related Work}
\vspace{-0.2cm}

\paragraph{Model Interpolation}
Model interpolation is a technique where the weights of multiple models are combined to produce new capabilities without any additional training~\citep{frankle2020linear,singh2020model,ainsworth2022git,yang:acm26}.
Several works~\citep{entezari2021role,nagarajan2019uniform} show that when two neural networks are initialized identically and share a part of the optimization trajectories, their \emph{modes} are linearly connected, enabling linear interpolation between them.
However, these explanations are limited to model interpolation of the same size.
In contrast, we aim to understand interpolation between models of different sizes. 
Recent work has explored zero-shot model size interpolation, which aims to interpolate between models of different sizes~\citep{cai2025llamaflex,kangaslahti2026boomerang}.
~\citet{cai2025llamaflex} show that explicitly training a Gumbel Softmax-based router allows them to interpolate between models of different sizes. 
More recently, ~\citet{kangaslahti2026boomerang} find that we can achieve zero-shot model size interpolation without training additional routers via knowledge distillation (Section \ref{sec:background}). 
We extend their work by comprehensively examining the user-facing choice of patching and its effects on model size interpolation. 

\looseness-1

\paragraph{Pruning}
Model pruning aims to compress model parameters by removing redundant parameters to improve inference efficiency while maintaining the full model performance~\citep{lecun1989optimal,han2015learning,sun2023simple,xia2023sheared,sreenivas2024llm,hou2025instruction}.
Analogous to student patching order is a model pruning technique called layer dropping, where the goal is to drop layers based on a computationally efficient criterion, such as activation similarity~\citep{men2024shortgpt,hinostroza2026rethinking} and statistical distance~\citep{zhang2024finercut}, so that the drop in performance is minimal. 
Our greedy KL patching algorithm (see Algorithm ~\ref{alg:greedy-patching-order}) closely relates to the iterative layer pruning algorithm in ~\citet{zhang2024finercut}. 
Although both rely on KL divergence (or a variant) between the larger and smaller models' token distributions to assess layer importance, ~\citet{zhang2024finercut} uses this information to remove layers, whereas we use it to patch layers and thereby construct interpolated models.
Furthermore, in Section \ref{sec:klpatch}, we show that the greedy KL divergence algorithm for patching often recovers a near-optimal patching trajectory.

\paragraph{Knowledge Distillation.}
Knowledge distillation is a well-known technique to create smaller models that mimic the behavior of larger models~\citep{hinton2015distilling,sanh2019distilbert,muralidharan2024compact,mansourian2025a}.
Prior work has used knowledge distillation to create compact language models such as DistilBERT~\citep{sanh2019distilbert} by training them on a corpus of text using the knowledge distillation objective~\citep{hinton2015distilling}.
~\citet{kangaslahti2026boomerang} showed that these distilled models exhibit zero-shot model size interpolation without any additional training.
We build on their initial experiment by creating interpolated models using DistilBERT and DistilGPT and patching the student in all possible ways to identify the optimal patching order.  
\looseness-1

\vspace{-0.2cm}
\section{Boomerang Distillation}\label{sec:background}
\vspace{-0.2cm}

 We provide an empirical case study building on \emph{boomerang distillation}~\citep{kangaslahti2026boomerang}, a recent technique for  model size interpolation. 
 Starting from a teacher LLM and a smaller student distilled from 
it, boomerang distillation produces a chain of intermediate-size models that smoothly interpolate from the student to the teacher without any additional training (Figure~\ref{fig:main-patching-overview}). The 
procedure has three steps, which we describe below: (1) student initialization, (2) knowledge distillation, and (3) student patching. \looseness-1

 \paragraph{Notation.}
 Let the teacher $\mT$ have $M$ transformer layers and the student $\mS$ from the same family have $N<M$ layers. Both models consist of an embedding layer, a sequence of transformer blocks, and an LM head, and
 all blocks produce hidden states of the same dimension.

\vspace{-0.2cm}
\paragraph{Step 1: Student Initialization.}\label{sec:boomerang_distillation:student_initialization}
 The student is initialized by partitioning the teacher's $M$ transformer layers into $N$ contiguous blocks; each student block is then initialized from the first teacher layer in the corresponding 
block~\citep{kangaslahti2026boomerang}. The student's embedding layer and LM head are copied from the teacher.
\looseness-1

\paragraph{Step 2: Knowledge Distillation and Alignment.}\label{sec:boomerang_distillation:knowledge_distillation}
 The initialized student is trained on a corpus such as the Pile~\citep{gao2020pile800gbdatasetdiverse} using a combination of knowledge distillation~\citep{hinton2015distilling} and a block-level alignment 
loss such as the cosine distance loss~\citep{sanh2019distilbert}, which encourages each student block to mimic the hidden state at the boundary of its corresponding teacher block. Concretely, given a token 
sequence $x=(x_1,\dots,x_L)$ and a position $j$, let $p_\mT(\cdot\mid x_{<j})$ and $p_\mS(\cdot\mid x_{<j})$ denote the teacher's and student's next-token distributions. The per-token loss is
 \begin{equation*}
 \mathcal{L}_{\vtheta_\mS}(x_j) = \mathrm{CE}(x_j\mid x_{<j};\vtheta_\mS) \;+\; \lambda_{\mathrm{KL}}\,\mathrm{KL}\!\left(p_\mT(\cdot\mid x_{<j})\,\big\|\,p_\mS(\cdot\mid x_{<j})\right) \;+\; 
\lambda_{\cos}\sum_{i=1}^{N}\mathcal{L}_{\cos}^{(i)}(x_{<j};\vtheta_\mT,\vtheta_\mS),
 \end{equation*}
 where $\mathrm{CE}$ is the cross-entropy loss against the ground-truth next token, $\mathcal{L}_{\cos}^{(i)}$ is the cosine distance between the student's hidden state at the boundary of block $i$ and the 
corresponding teacher hidden state, and $\lambda_{\mathrm{KL}},\lambda_{\cos}>0$ weigh the distillation and alignment terms.

\paragraph{Step 3: Student Patching.}\label{sec:boomerang_distillation:student_patch}

 After training, intermediate-size models are constructed by \emph{patching}---replacing student layers with the corresponding teacher blocks. For each subset $A\subseteq\{1,\dots,N\}$, we write $\mM_A$ for 
the model obtained by patching in the teacher blocks aligned with the indices in $A$, leaving the remaining $N-|A|$ student blocks intact. The endpoints recover the student and teacher, $\mS=\mM_\emptyset$ and
 $\mT=\mM_{\{1,\dots,N\}}$, and growing $A$ one block at a time produces a chain of intermediate models linking the two without any additional training (Figure~\ref{fig:main-patching-overview}). The order in 
which student blocks are patched determines the trajectory through this chain, and is the central object of study in the remainder of the paper.
\looseness-1

 \section{Layer Patching as a Shortest-Path Problem}\label{sec:model-patching}
 \vspace{-0.2cm}
 First, we formalize student patching as a combinatorial optimization problem (Section~\ref{sec:interpolation-optimization}). We then recast the 
optimization as a shortest-path problem on a KL-weighted graph (Section~\ref{sec:patching-orders}), motivating the KLPatch algorithm of Section~\ref{sec:klpatch}.

 \subsection{A Combinatorial Optimization Perspective}\label{sec:interpolation-optimization}
 \vspace{-0.2cm}
 Boomerang distillation requires choosing \emph{which} student layers to patch with teacher blocks at each interpolated size, and we observe in Section~\ref{sec:patching-order-permutations} that this choice strongly determines downstream performance. We make this choice precise as a combinatorial optimization problem.
 
 The choice of which layers to patch decomposes by interpolated size: for each cardinality $k \in \{1, ..., N\}$, there are $\binom{N}{k}$ candidate subsets $A \subseteq \{1, ..., N\}$, each inducing a partially patched model $\mM_A$, and the best one depends on the chosen quality metric. 
 Formally, fix an evaluation function $f_{\mS, \mT}: 2^{\{1, ..., N\}}\to\mathbb{R}$ that scores each subset $A$ by the quality of $\mM_A$, with higher meaning better.\footnote{If lower is better, as in the case of perplexity, $f_{\mS,\mT}$ is negated to recover a maximization problem.} 
 The best $k$-subset then maximizes $f_{\mS, \mT}(A)$ subject to $|A|=k$; collecting these maximizers across all $k$ defines an \textit{optimal interpolation} between $\mS$ and $\mT$.
 
 \begin{problem}[Optimal Interpolation]\label{prob:optimal-interpolation}
     Given a student $\mS$, a teacher $\mT$, and an objective function $f_{\mS, \mT}: 2^{\{1, ..., N\}}\to\mathbb{R}$, the \textit{optimal interpolation} between $\mS$ and $\mT$ with respect to $f_{\mS,\mT}$ 
is the sequence of models $\mS=\mM_{A_0}, \mM_{A_1}, ..., \mM_{A_{N-1}}, \mM_{A_N}=\mT$ defined by
     \begin{equation*}
         A_k = \argmax_{A\subseteq \{1,...,N\},\,|A|=k} f_{\mS, \mT}(A).
     \end{equation*}
 \end{problem}
 
 For simplicity, we define Problem~\ref{prob:optimal-interpolation} over cardinality instead of model size, but an equivalent definition taking size into account could be used as well. 
 As $\binom{N}{k}$ subsets exist for each $k$, finding the optimal interpolation requires evaluating all $\sum_{k=1}^{N-1}\binom{N}{k}=2^N-2$ proper subsets of $\{1,\ldots,N\}$, which is infeasible in practice
 since each evaluation of $f_{\mS,\mT}$ requires evaluating an interpolated model. 
 
An alternative formulation searches over permutations of layers rather than subsets, replacing the question 
``which $k$ layers do we patch, for all $k$?'' with ``in what order do we patch all $N$ layers?'' A permutation of layers naturally gives rise to an interpolation of models by patching layers sequentially according to the permutation.

An optimal permutation $\pi^*$ should also take the incremental costs $d_{\pi^*(k)}$ of patching the $k$th layer into account, and maximize the \textbf{area under the interpolation curve (AUIC)} of $f_{\mS, \mT}$, which we define as

\begin{equation*}
    \mathrm{AUIC}_f(\pi) \triangleq \sum_{k=1}^{N} d_{\pi(k)}\frac{f_{\mS, \mT}(\mM_{A^\pi_k})+f_{\mS, \mT}(\mM_{A^\pi_{k-1}})}{2}, 
\end{equation*}
where $A^\pi_k \triangleq \{\pi(i) : i \le k\}$.
 
 \begin{problem}[Optimal Permutation]\label{prob:optimal-permutation}
     Given a student $\mS$, a teacher $\mT$, and an objective function $f_{\mS, \mT}: 2^{\{1, ..., N\}}\to\mathbb{R}$, the \textit{optimal permutation} $\pi^*\in\mathfrak S_N$ of layers of $\mS$ with respect 
to $f_{\mS, \mT}$ is
     \begin{equation*}
         \pi^* = \argmax_{\pi\in\mathfrak S_N} \mathrm{AUIC}_f(\pi).
     \end{equation*}
 \end{problem}

\paragraph{Computational complexity.}

\begin{wraptable}{r}{0.5\textwidth}
    \centering
    \vspace{-10pt}
    \resizebox{\linewidth}{!}{%
    \begin{tabular}{ccc}
        \toprule
        \textbf{Problem} & \makecell{\textbf{Evaluation}\\ \textbf{Complexity}} & \makecell{\textbf{Search Space}\\ \textbf{Size}} \\
        \midrule
        \makecell{Problem \ref{prob:optimal-interpolation}\\ (Optimal Interpolation)} & \makecell{$2^N-2$ evaluations \\each costing $O(f_{\mS, \mT})$} & $2^N-2$ \vspace{5pt} \\
        \makecell{Problem \ref{prob:optimal-permutation}\\ (Optimal Permutation)}& \makecell{$2^N-2$ evaluations \\each costing $O(f_{\mS, \mT})$} & $N!$\\
        \bottomrule
    \end{tabular}}
    \caption{\textbf{Computational complexity of Problems \ref{prob:optimal-interpolation} and \ref{prob:optimal-permutation}.} Both problems require the same number of evaluations of $f_{\mS, \mT}$, but Problem \ref{prob:optimal-permutation} has a higher exhaustive search complexity over the evaluated models. However, evaluations of $f_{\mS, \mT}$ are significantly more computationally expensive, since they require evaluating an interpolated model (often through inference), while each evaluation in the exhaustive search consists of comparing or summing $O(N)$ real numbers.\looseness=-1}
    \vspace{-10pt}
    \label{tab:computational-complexity}
\end{wraptable}
The permutation search space is in fact larger than the interpolation search space ($|\mathfrak S_N|=N!$ versus $2^N-2$). 
However, we note that the computational complexity of both problems is similar.
This is because the computationally expensive part is typically evaluating $f_{\mS, \mT}$, as for most evaluation functions this requires computing inference passes over an interpolated model. 
For both Problem \ref{prob:optimal-interpolation} and \ref{prob:optimal-permutation}, $f_{\mS, \mT}$ needs to  be evaluated for all of the possible $2^{N}-2$ interpolated models.
Once these $f_{\mS, \mT}$ values are obtained, searching over interpolations or permutations is relatively cheap in comparison.
Searching over interpolations corresponds to computing the maximum over $f_{\mS, \mT}$ for each model size, or computing a maximum over $\binom{N}{k}$ real numbers for $k \in \{1, \ldots, N\}$.
Similarly, performing a search over permutations requires computing the sum over $f_{\mS, \mT}(A)$ for each permutation, or computing a sum over $N$ real numbers for $N!$ permutations.
Thus, finding the optimal permutation has effectively the same computational complexity as finding the optimal interpolation, as any $\mathrm{AUIC}_f(\pi)$ can be quickly evaluated once $f_{\mS, \mT}(A)$ is known for all $A\subset \{1, ..., N\}$. 
The computational complexity of both problems is summarized in Table~\ref{tab:computational-complexity}. 

We will see in Section~\ref{sec:patching-orders} 
that Problem~\ref{prob:optimal-permutation} admits a clean shortest-path interpretation with theoretical guarantees. 
More immediately, both formulations admit a natural greedy attack: grow the patched set one layer at a time, a strategy that underlies many tractable heuristics in combinatorial optimization, in some cases with formal approximation guarantees, e.g., for submodular maximization~\citep{nemhauser1978analysis}. 
Applied to either Problem~\ref{prob:optimal-interpolation} or Problem~\ref{prob:optimal-permutation}, greedy addition produces the same object---a nested chain of subsets, equivalently a permutation---so the two formulations collapse algorithmically under this lens, though they differ in scoring: Problem~\ref{prob:optimal-interpolation} evaluates each $A_k$ separately while Problem~\ref{prob:optimal-permutation} aggregates across $k$. 
We instantiate this strategy in 
Section~\ref{sec:klpatch} (KLPatch). 
Empirically, we also observe that solutions to Problem~\ref{prob:optimal-permutation} closely track solutions to Problem~\ref{prob:optimal-interpolation} in terms of perplexity and downstream accuracy (Appendix~\ref{apx:patching-order-vs-interpolation}).

For $f_{\mS, \mT}=\mathrm{PPL}_{\text{data}}$, we will furthermore write $\mathrm{AUPIC}(\pi)\triangleq\mathrm{AUIC}_f(\pi)$ for the \textbf{area under the perplexity curve (AUPIC)} w.r.t. data-perplexity $\mathrm{PPL}_{\text{data}}$.
A log analogue $\mathrm{AUPIC}^{\log}(\pi)$, which replaces a sum over perplexities $\mathrm{PPL}_{\mathrm{data}}$ with a sum over log-perplexities $\log \mathrm{PPL}_{\mathrm{data}}$, will be useful in 
Section~\ref{sec:patching-orders}. Replacing data-perplexity with perplexity against samples drawn from $\mT$ yields teacher-relative variants $\mathrm{AUPIC}_\mT(\pi)$ and $\mathrm{AUPIC}^{\log}_\mT(\pi)$; 
these underpin the shortest-path characterization of Section~\ref{sec:patching-orders} (Proposition~\ref{prop:paths}).

\subsection{Patching Orders as Shortest Paths}\label{sec:patching-orders}
\vspace{-0.2cm}
Layer patching is typically done incrementally: each step adds one previously unpatched layer to the current patched set. The reachable configurations therefore form a Boolean lattice on $\{1,\ldots,N\}$, and patching
 orders correspond to directed paths through it from $\mS$ (no layers patched) to $\mT$ (all layers patched). AUPIC and its log/teacher variants are themselves sums of a per-state evaluation along such a path,
 so Problem~\ref{prob:optimal-permutation} reduces to a shortest-path problem on a suitably weighted version of the lattice. The remainder of this subsection formalizes this view and identifies the objective 
for which the reduction is exact.

 \begin{definition}[Permutation Graph]\label{def:interpolation-graph}
 Let $\mathcal{G}$ be the Boolean lattice on $\{1,\ldots,N\}$, viewed as a directed acyclic graph: vertices are subsets $A\subseteq\{1,\ldots,N\}$, identified with the partially-patched models $\mM_A$; an edge
 connects $\mM_A$ to $\mM_{A\cup\{i\}}$ whenever $i\notin A$. The graph has $\binom{N}{k}$ vertices at depth $k$, with root $\mM_\emptyset = \mS$ and sink $\mM_{\{1,\ldots,N\}} = \mT$. Given a calibration 
dataset $\mathcal{D}_{\mathrm{cal}}$ and writing
 \begin{equation*}
     K(A)\triangleq\mathbb{E}_{x\sim \mathcal{D}_{\mathrm{cal}}}\,\mathrm{KL}\bigl(p_\mT(\cdot \mid x)\,\big\|\, p_{\mM_{A}}(\cdot\mid x)\bigr),
 \end{equation*}
assign each edge the weight\looseness-1
 \begin{equation*}
     w\bigl(\mM_A\rightarrow \mM_{A\cup \{i\}}\bigr)\triangleq d_i\frac{K(A)+K(A\cup \{i\})}{2}.
 \end{equation*}
 \end{definition}

  Each patching order $\pi\in\mathfrak{S}_N$ corresponds to a directed path $\gamma_\pi$ of length $N$ from $\mS$ to $\mT$ in $\mathcal{G}$, traversing the intermediate models $\{\mM_{A^\pi_k}\}_{k=0}^{N}$. Its
 total weight is
 \[
 E(\pi)\triangleq\sum_{e\in \gamma_\pi} w(e)
 \;=\; \sum_{k=0}^{N-1} d_{\pi(k)}\frac{K(A^\pi_k)+K(A^\pi_{k+1})}{2}.
 \]
Shortest paths in $\mathcal{G}$ are precisely the optimal permutations from Problem~\ref{prob:optimal-permutation} with $f_{\mS,\mT}=\log\mathrm{PPL}_{\mT}$.

 \begin{proposition}[Shortest KL Paths are optimal]\label{prop:paths}
 The shortest paths in \(\mathcal G\) are exactly the optimal permutations w. r. t. \(\mathrm{AUPIC}^{\log}_\mT(\pi)\) (to be understood as a set-equality if there are multiple minimizers):
 \[
 \arg\min_{\pi\in\mathfrak S_N} E(\pi)
 \;=\;
 \arg\min_{\pi\in\mathfrak S_N} \mathrm{AUPIC}^{\log}_{\mT}(\pi).
 \]
 \end{proposition}
 
Proposition~\ref{prop:paths} thus reduces the combinatorial search Problem~\ref{prob:optimal-permutation} to a shortest-path problem on $\mathcal{G}$. Two assumptions distinguishing the proposition's setting from the metrics reported in our experiments—perplexity referenced to $\mT$ rather than to data, and on the log scale rather than the raw scale—are discussed in Appendix~\ref{apx:proofs}: 
Proposition~\ref{prop:reference-distribution} bounds the data-vs-teacher gap, and Proposition~\ref{prop:log-vs-raw} the log-vs-raw gap. 
We note that in our experiments in Sections~\ref{sec:patching-order-permutations} and \ref{sec:klpatch}, we report a normalized version of the AUPIC; this amounts to rescaling by a constant factor and does not affect the optimality guarantee in Proposition~\ref{prop:paths}. Empirically, on the models we study, path length in $\mathcal{G}$ strongly correlates with 
data-AUPIC (Appendix~\ref{apx:kl-vs-perplexity}, Figures~\ref{fig:kl-path-aupic-correlation}~and~\ref{fig:kl-path-aupic-correlation-pile}). Section~\ref{sec:klpatch} exploits this characterization with a tractable greedy algorithm which finds paths in $\mathcal{G}$ that are often near-optimal.

While $\mathcal{G}$ could be populated with different choices of edge weights, such as perplexity or downstream accuracy, using KL to the teacher model provides several advantages: Proposition~\ref{prop:paths} draws a direct connection to the teacher-relative log-AUPIC, while computing the KL, unlike computing $\mathrm{AUPIC}^{\log}_{\mT}$, does not require sampling from $\mT$. 
Unlike downstream accuracy, KL is task-agnostic and provides a dense signal over the full predictive distribution. 
Finally, KL measures whether a patched model remains aligned to the teacher, unlike perplexity (however, we do find it strongly correlates with perplexity; see Appendix~\ref{apx:kl-vs-perplexity}, Figures~\ref{fig:kl-aupic-correlation} and \ref{fig:kl-aupic-correlation-pile}). 
In practice, KL performed better than cosine distance (Appendix \ref{apx:klpatch-metric-ablation}), which correlates poorly with downstream performance~\citep{zhang2024finercut,hinostroza2026rethinking}.

\vspace{-0.2cm}
 \section{Empirical Characterization of Layer Patching}\label{sec:patching-order-permutations}
 \vspace{-0.2cm}
 Section~\ref{sec:model-patching} cast the choice of patching order as a combinatorial optimization with a clean shortest-path characterization. We now characterize this optimization landscape empirically: 
Section~\ref{sec:experiment:distilbert} performs a full sweep over all possible interpolations on DistilBERT and DistilGPT2; Section~\ref{sec:experiment:similarities} relates AUPIC to Spearman's footrule 
distance from the optimum; and Section~\ref{sec:experiment:patching_order} samples and evaluates a subset of patching orders for larger LLMs, where full enumeration is infeasible.
 Across all three regimes, last-to-first patching is a consistently strong but not always optimal recipe, and the gap between naive baselines and the empirical optimum motivates the algorithm of 
Section~\ref{sec:klpatch}.

\looseness-1

\vspace{-0.2cm}
 \subsection{Full Sweep on DistilBERT and DistilGPT2}\label{sec:experiment:distilbert}
 We start with DistilBERT and DistilGPT2 because their small size ($N=6$ layers each) makes the entire space of $720$ patching orders—and indeed all $2^6=64$ partially-patched models—directly enumerable, 
allowing us to characterize interpolation behavior exhaustively rather than through sampling.\looseness-1

 \paragraph{Setup.}
 We use DistilBERT and DistilGPT2 \citep{sanh2019distilbert} as students, with the corresponding teacher checkpoints, BERT \citep{devlin2019bert} and GPT2 \citep{radford2019language}, all from Hugging Face. 
Both students have $N=6$ layers, yielding $6!=720$ patching orders each. For every order, we iteratively patch one layer at a time, evaluating pseudo-perplexity (DistilBERT) or perplexity (DistilGPT2) on 
Wikitext \citep{wikitext} at each of the $N+1$ interpolation points; aggregating across the trajectory gives AUPIC. We report the order with the lowest AUPIC as the empirical optimum.\looseness-1

\paragraph{Results.}

\begin{figure*}[!t]
  \centering
  \includegraphics[width=0.8\linewidth]{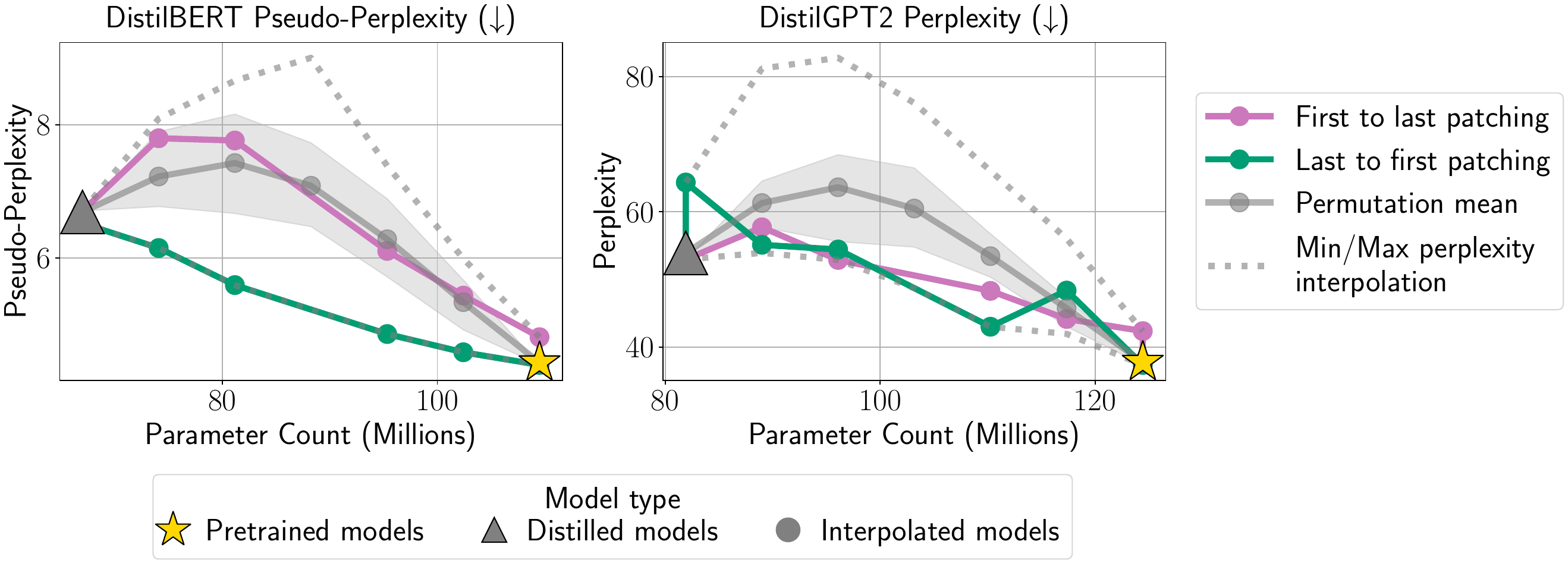}
  \vspace{-0.2cm}
   \caption{\textbf{All permutations of patching order for DistilBERT and DistilGPT2.} Patching order significantly affects interpolation performance, and the best ordering yields nearly linear interpolation 
in (pseudo-)perplexity between the students (DistilBERT, DistilGPT2) and teachers (BERT, GPT2). Shaded bands show the (25-75) inter-quartile range over all $720$ orderings.}
  \vspace{-0.3cm}
  \label{fig:all-distilbert-distilgpt}
\end{figure*}

 Patching order has a substantial impact on interpolation performance (Figure~\ref{fig:all-distilbert-distilgpt}, Appendix Figure~\ref{fig:aupic-distilbert-distilgpt}). The best ordering yields near-linear interpolation 
between the student and teacher in pseudo-perplexity (DistilBERT) or perplexity (DistilGPT2), while the worst ordering produces only two interpolated models that improve on the student. The last-to-first 
ordering of \citet{kangaslahti2026boomerang} attains optimal AUPIC on DistilBERT and near-optimal AUPIC on DistilGPT2, whereas first-to-last is less consistent: near-optimal on DistilGPT2 but above the mean on
 DistilBERT. The dotted line envelope in Figure~\ref{fig:all-distilbert-distilgpt} bounds the best and worst interpolation curves over all $2^6=64$ partially-patched models, including non-nested ones not reachable 
by any single patching order; gaps between this envelope and the best ordering quantify the headroom available from relaxing the ordering constraint to general subset-based patching. The substantial gap 
between the optimum and the bulk of the AUPIC distribution, together with the inconsistency of fixed orderings across models, motivates a principled procedure for selecting patching orders, which we develop in
 Section~\ref{sec:klpatch}.

\vspace{-0.2cm}
 \subsection{Good Patching Orders Cluster Around the Optimum}\label{sec:experiment:similarities}
 \vspace{-0.2cm}
 Section~\ref{sec:experiment:distilbert} characterized the spread of AUPIC across all orderings. Here we ask a finer question: are good orderings \emph{near} the optimum in permutation space, or are they 
scattered? A clean structural answer would say that small perturbations of the optimum remain near-optimal, suggesting a smooth landscape amenable to local search.

 \vspace{-0.2cm}
\paragraph{Setup.}
 We use the $720$ trajectories from Section~\ref{sec:experiment:distilbert} for both DistilBERT and DistilGPT2 and measure each ordering's Spearman footrule distance to the empirical optimum. For permutations 
$\pi,\pi'\in\mathfrak{S}_N$, the footrule distance is computed as $d(\pi,\pi') = \sum_{i=1}^{N}\bigl|\pi(i)-\pi'(i)\bigr|,$ i.e., the $\ell_1$ distance between the position vectors. We then bin orderings by their distance to the optimum and report the mean AUPIC per bin.

\begin{figure*}[!t]
  \centering
  \includegraphics[width=0.9\linewidth]{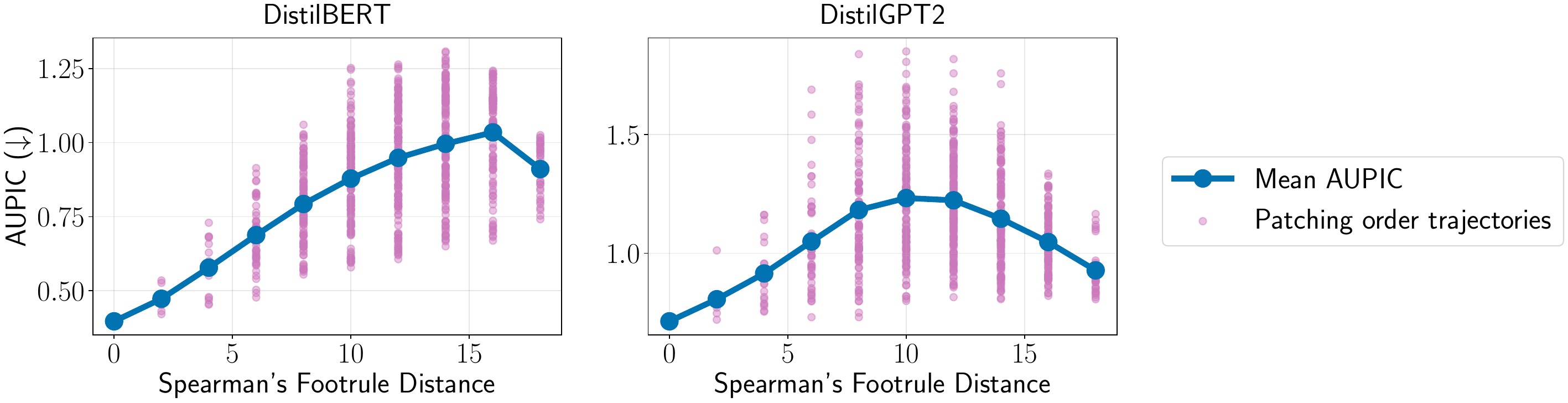}
   \caption{\textbf{AUPIC vs.~Spearman footrule distance from the minimum-AUPIC permutation.} On both DistilBERT and DistilGPT2, mean AUPIC rises with distance from the optimum and then flattens, indicating a cluster of near-optimal orderings around the empirical minimum.}
  \vspace{-0.2cm}
  \label{fig:permutation-distance-distilbert-distilgpt}
\end{figure*}

 \vspace{-0.2cm}
\paragraph{Results.}
Mean AUPIC increases with footrule distance from the optimum and then plateaus at the largest distances (Figure~\ref{fig:permutation-distance-distilbert-distilgpt}). Concretely, orderings within small 
footrule distance to the optimum form a cluster of near-optimal AUPIC, while at larger distances the relationship becomes noisy and AUPIC saturates. In the language of Section~\ref{sec:patching-orders}, 
near-optimal patching orders correspond to short paths in $\mathcal{G}$ that share a substantial number of edges with the shortest path; this local structure is what makes greedy approaches such as KLPatch 
(Section~\ref{sec:klpatch}) viable: small, locally-informed deviations from the shortest-path skeleton tend not to leave the near-optimal basin.

\vspace{-0.2cm}
 \subsection{Effect of Layer Patching in Large Language Models}\label{sec:experiment:patching_order}
 \vspace{-0.2cm}
 We now ask whether the small-model findings of Section~\ref{sec:experiment:distilbert} carry over to larger language models. The full enumeration used there ($N!=720$) is infeasible at this scale ($N! > 10^{12}$), so we 
instead sample $200$ random patching orders per model and study the resulting AUIC distribution.

 \paragraph{Setup.}
 We use the distilled student models from \citet{kangaslahti2026boomerang} for Qwen3-4B Base, Qwen3-8B Base, and Pythia-6.9B. For each of $200$ 
randomly sampled patching orders, we iteratively patch one student layer at a time and evaluate downstream accuracy at every interpolation point, averaging across the task suite of 
\citet{kangaslahti2026boomerang} to obtain an AUIC value per ordering (see Appendix \ref{apx:experimental_details} for details). We additionally include first-to-last and last-to-first as fixed-order baselines. Wikitext AUPIC results for the same models are reported 
in Appendix~\ref{apx:additional-permutation}.\looseness-1

\begin{figure*}[!t]
  \centering
  \includegraphics[width=0.9\linewidth]{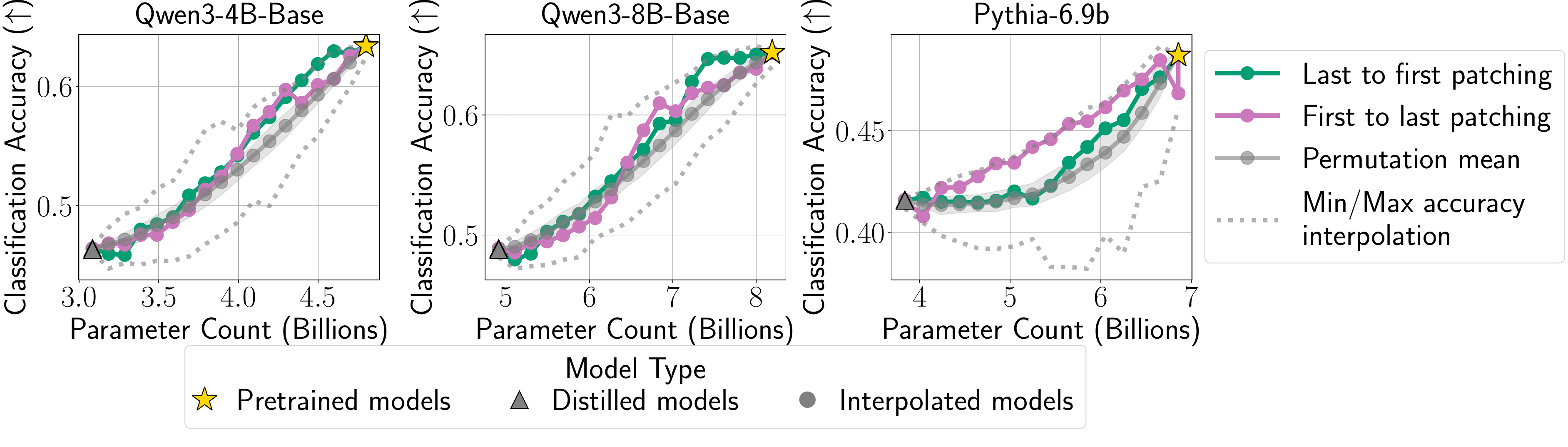}
  \vspace{-0.2cm}
     \caption{\textbf{Downstream interpolation curves across patching orders.} For each of Qwen3-4B, Qwen3-8B, and Pythia-6.9B, the mean downstream accuracy is plotted against patched model size, 
across the $200$ sampled orderings together with the first-to-last and last-to-first baselines. Shaded bands show the (25-75) inter-quartile range for the $200$ orderings. First-to-last remains competitive but does not always match the best-sampled ordering.}
  \vspace{-0.5cm}
  \label{fig:permutation-comparison-llms}
\end{figure*}

 \paragraph{Results.}
 Across all three models, patching order has a substantial impact on downstream interpolation quality (Figure~\ref{fig:permutation-comparison-llms}, Appendix Figure~\ref{fig:permutation-comparison-llms-distribution}). The 
small-model regularity, however, does not transfer cleanly: while last-to-first attains the lowest \emph{Wikitext} AUPIC (Appendix~\ref{apx:additional-permutation}), it leaves a substantial gap in downstream 
AUIC relative to the best-sampled ordering. This is consistent with prior observations that pretraining perplexity is an imperfect proxy for downstream performance~\citep{liu2023same}. Moreover, the IQR across
 the $200$ sampled orderings is narrow relative to this gap, so the lift available from finding a near-best order—rather than averaging over random ones—is large in practice; Section~\ref{sec:klpatch} 
develops a procedure for doing so.\looseness-1

\vspace{-0.2cm}
 \section{KL-Guided Patching Orders}\label{sec:klpatch}

Section~\ref{sec:patching-order-permutations} showed that patching order has a substantial impact on interpolation quality, and that fixed orderings (first-to-last, last-to-first) leave a large headroom 
relative to the empirical optimum. We now turn this characterization into an algorithm.

\subsection{The KLPatch Algorithm}\label{sec:klpatch-algo}

KLPatch (Algorithm~\ref{alg:greedy-patching-order}) is an iterative greedy procedure that, at each step, patches the student layer whose patched version minimizes KL
 divergence to the teacher on a calibration set---i.e., it greedily picks the cheapest outgoing edge in the interpolation graph $\mathcal{G}$ of Definition~\ref{def:interpolation-graph}. It repeats until all $N$ layers are patched. Since iteration $k$ evaluates $N-k$ candidates, KLPatch reduces the search from $O(2^N)$ exhaustive evaluations to $O(N^2)$ KL evaluations. Reading 
Problem~\ref{prob:optimal-permutation} through Proposition~\ref{prop:paths}, KLPatch is a greedy approximation to the shortest-path problem on $\mathcal{G}$. 
In empirical evaluations (Section~\ref{sec:experiment:klpatch}), we show that it yields permutations close to the optima of both Problem~\ref{prob:optimal-interpolation} and Problem~\ref{prob:optimal-permutation}.\looseness=-1

\begin{algorithm}[h]
 \caption{KLPatch: Iterative Greedy Construction of a Patching Order for Model Interpolation}
\label{alg:greedy-patching-order}
\small
\begin{algorithmic}[1]
\Require Teacher model $\mT$ with layer blocks $\mathcal{B}=(\mathbf{b}^{(1)},\ldots,\mathbf{b}^{(N)})$; trained student $\mS$; calibration set $\mathcal{D}_{\mathrm{cal}}$.%
\State $U \gets \{1,\ldots,N\};\quad \mM \gets \mS;\quad \pi \gets [\ ]$
\While{$U \neq \emptyset$} \Comment{Iteratively select the next layer to patch.}
    \For{each $i \in U$} \Comment{Patch the student model and measure the KLDivergence.}
        \State $\tilde{\mM}^{(i)} \gets \textsc{Patch}(\mM, i, \mathbf{b}^{(i)})$ \Comment{replace $i$th layer in $\hat{\mS}$ with $\mathbf{b}^{(i)}$}
        \State $\ell_i \gets \mathbb{E}_{x\sim\mathcal{D}_{\mathrm{cal}}}\,\mathrm{KL}(p_\mT(\cdot\mid x)\|p_{\tilde{\mM}^{(i)}}(\cdot\mid x))$
    \EndFor
    \State $i^\star \gets \arg\min_i \ell_i$ \Comment{Select the student layer with lowest KL to be patched.}
    \State $\mM \gets \textsc{Patch}(\mM, i^\star, \mathbf{b}^{(i^\star)})$ 
    \State $\pi \gets \pi \mathbin{\|} [i^\star]$\Comment{Add the student layer index to the patching order sequence $\pi$.}
    \State $U \gets U \setminus \{i^\star\}$
\EndWhile

\State \Return $\pi$
\end{algorithmic}
\end{algorithm}
\vspace{-0.3cm}

\subsection{KLPatch Finds Near-Optimal Patching Orders}\label{sec:experiment:klpatch}

 \paragraph{Setup.} We compare KLPatch's patching order against first-to-last / last-to-first baselines as well as the best interpolation and permutation obtained from $200$ random orderings on the same models, students, and downstream task suite as in Section~\ref{sec:experiment:patching_order} (Qwen3-4B, 
Qwen3-8B, Pythia-6.9B), and against the distribution of $200$ random orderings from that section. For KLPatch's calibration set $\mathcal{D}_{\mathrm{cal}}$, we use $64$ samples drawn from the 
Pile~\citep{gao2020pile800gbdatasetdiverse}; the same calibration set is used across all student layers and all iterations of Algorithm~\ref{alg:greedy-patching-order}. We additionally evaluate KLPatch on 
Llama-3.2-3B in Appendix~\ref{apx:llama}. In Appendix \ref{apx:klpatch-original-boomerang-baselines}, following ~\citet{kangaslahti2026boomerang}, we compare KLPatch to additional baselines such as layer pruning and randomly initializing the student model during distillation.\looseness=-1

\begin{figure*}[!t]
  \centering
  \includegraphics[width=0.9\linewidth]{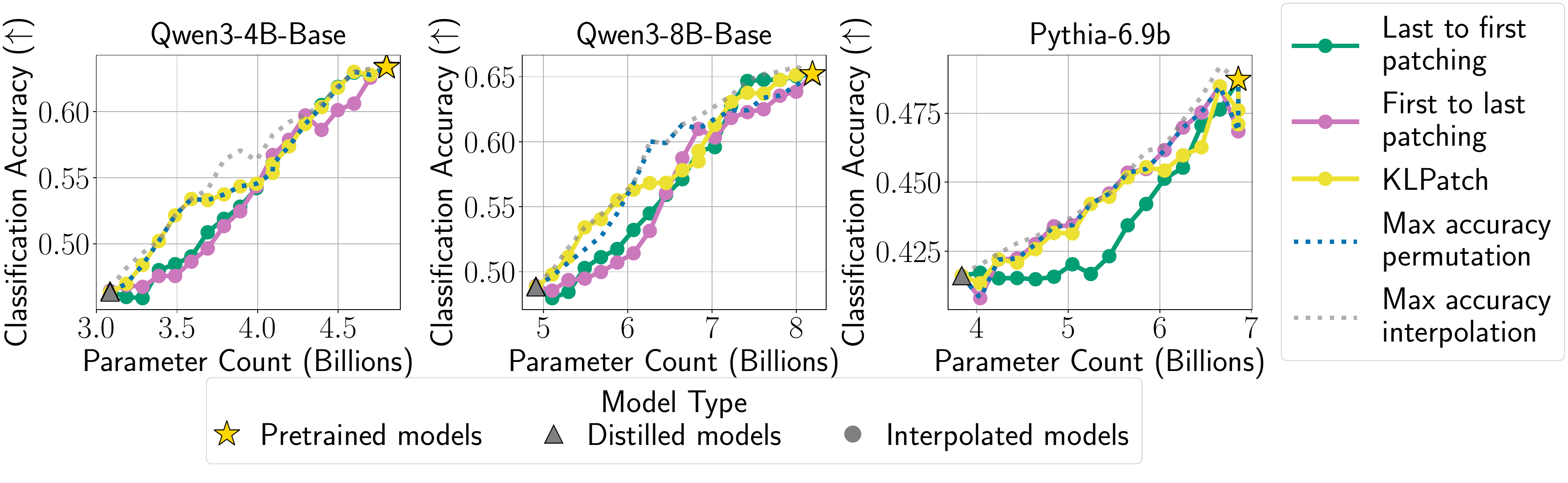}
  \vspace{-0.2cm}
   \caption{\textbf{KLPatch interpolation curves on Qwen3-4B, Qwen3-8B, and Pythia-6.9B.} For each model, mean downstream accuracy is plotted against the patched student size, comparing KLPatch 
(Algorithm~\ref{alg:greedy-patching-order}) against first-to-last and last-to-first baselines. Across all three models, KLPatch tracks or outperforms last-to-first on both classification and generation tasks.}
  \vspace{-0.5cm}
  \label{fig:klpatch-main}
\end{figure*}

 \paragraph{Results.}
 KLPatch yields strong model interpolations across various families and tasks. On Qwen3-4B, Qwen3-8B, and Pythia-6.9B, KLPatch outperforms or matches both first-to-last and last-to-first on classification (Figure~\ref{fig:klpatch-main}), generation (Appendix~\ref{apx:klpatch-generation}), and 
Wikitext perplexity (Appendix~\ref{apx:klpatch-perplexity}). 
It also beats all of the $200$ random orderings for Qwen3-4B-Base and Pythia-6.9b and all but one of the random orderings for Qwen3-8B-Base (Table~\ref{tab:klpatch_baselines_auic}), yielding near-best performance out of the sampled interpolation trajectories on these three models. 
In Appendix~\ref{apx:llama}, we also evaluate KLPatch on Llama models, and show that while it does not work out-of-the-box, slightly modifying the algorithm by always patching layer 1 first recovers its strong performance.
We note that Llama was shown to be a special case and required different initialization and patching order for boomerang distillation~\citep{kangaslahti2026boomerang}, indicating that the effect of patching order is influenced by the student model initialization.
Overall, KLPatch provides an efficient approach for reliably creating high-performing model interpolations.

\subsection{KLPatch for Variable Quantization}\label{sec:experiment:klpatch-quantization}

\begin{figure*}[!t]
  \centering
  \includegraphics[width=0.9\linewidth]{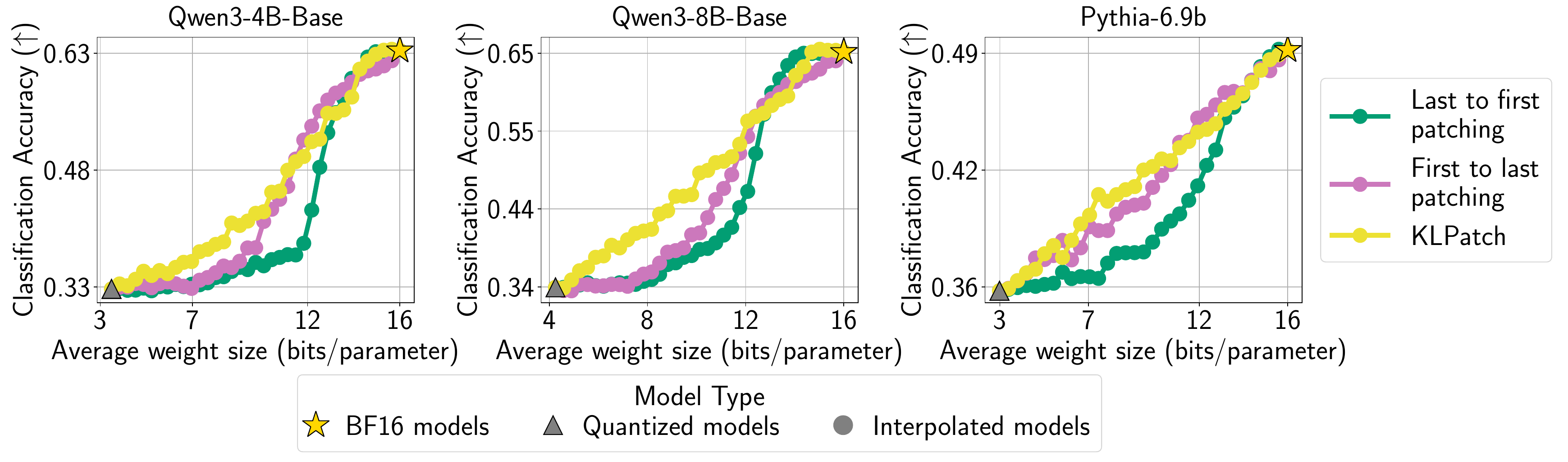}
  \vspace{-0.2cm}
   \caption{\textbf{KLPatch interpolation curves for variable quantization.} Mean downstream accuracy is plotted against bits per parameter, comparing KLPatch 
(Algorithm~\ref{alg:greedy-patching-order}) against first-to-last and last-to-first baselines. Across all three models, KLPatch tracks or outperforms sequential patching orders.\looseness=-1}
  \vspace{-0.4cm}
  \label{fig:klpatch-quantized}
\end{figure*}

\paragraph{Setup.} 
We apply KLPatch and sequential patching orders directly to the setting of variable quantization by interpolating between quantized and bf16 LLMs. 
We use Qwen3-4B Base, Qwen3-8B Base, and Pythia-6.9B at bf16 precision as $\mT$ and construct $\mS$ by compressing each model with 2-bit GPTQ quantization~\citep{frantar2023gptq}.
Similarly to \cite{dumitru2025variable}, we use 2-bit quantization as a testbed for $\mS$ since we do not observe a significant drop in performance for 4-bit or 8-bit precision models.
We note that the embedding layer and LM head are kept full-precision for compatibility.
Starting from the quantized model $\mS$, we can increase the bits per parameter of the model by patching any quantized layer $k$ in $\mS$ with the $k$th bf16 layer in $\mT$.
We use the same calibration set $\mathcal{D}_{\mathrm{cal}}$ from Section \ref{sec:experiment:klpatch} of $64$ samples drawn from the Pile~\citep{gao2020pile800gbdatasetdiverse} and evaluate the interpolated models on the classification task suite from Section~\ref{sec:experiment:patching_order}.\looseness=-1

\paragraph{Results.} 
KLPatch produces interpolations that outperform or match first-to-last and last-to-first patching orders when constructing models with variable quantization (Figure \ref{fig:klpatch-quantized}). 
KLPatch especially compares favorably to sequential strategies in the region with small average bits per parameter, indicating that it is particularly beneficial for constructing highly quantized models. 
These results provide empirical evidence that our shortest-paths formulation and KLPatch algorithm, which are theoretically applicable to any per-layer intervention method, generalize beyond boomerang distillation to other model interpolation contexts.
Applying our optimization problem and KLPatch to new settings, for example in model merging via layer swapping~\citep{bandarkar2025layerswappingzeroshotcrosslingual, shin2026layerwiseswappinggeneralizablemultilingual}, is a promising direction for future research.

\section{Conclusion}
\vspace{-0.2cm}
 We presented a systematic study of layer patching in boomerang distillation. An exhaustive sweep over all interpolations on DistilBERT and DistilGPT2 shows that last-to-first is globally optimal or near-optimal, and that good orderings cluster together in permutation space. On Qwen3-4B, Qwen3-8B, and Pythia-6.9B, last-to-first remains 
competitive, but the gap to the empirical optimum is significant. To close this gap, we cast finding the optimal patching order
 as a shortest-path problem on a KL-weighted lattice (Proposition~\ref{prop:paths}), and propose KLPatch, a greedy $O(N^2)$ algorithm that commits to the cheapest outgoing edge at each step; KLPatch produces the best or near-best patching orders and consistently 
outperforms first-to-last and last-to-first on Qwen and Pythia.
We also show that KLPatch can be applied to other model interpolation settings such as variable quantization.
 
Yet, several questions remain open. The magnitude of KLPatch's gains varies across model families, and characterizing the conditions under which patching-order optimization helps most is a natural next step. Other
 directions include sub-greedy alternatives that exploit Proposition~\ref{prop:paths} more globally, amortized procedures that avoid recomputing teacher-conditioned KL at every step, and reinforcement-learning
 approaches in the spirit of \citet{he2018amc}, which adaptively select interpolation trajectories rather than committing greedily.

\section*{Acknowledgments}
David Alvarez-Melis, Sara Kangaslahti, Jonathan Geuter, and Nihal V. Nayak acknowledge support from the National Science Foundation Graduate Research Fellowship (Grant No. DGE 2140743), the Kempner Institute, FAS Dean's Competitive Fund for Promising Scholarship, Aramont Fellowship Fund, and the NSF AI-SDM Institute (Grant No. IIS-2229881). 
Jonathan Geuter is supported by a Kempner Institute Graduate Fellowship.
Francesco Locatello's contribution to this research was funded in part by the Austrian Science Fund (FWF) 10.55776/COE12.

\bibliographystyle{plainnat}
\bibliography{bibliography}

\clearpage
{
  \hypersetup{linkcolor=black}
  \begingroup
  \renewcommand{\baselinestretch}{0.879}\normalsize
  \tableofcontents
  \endgroup
}
\clearpage
\appendix
\section{Limitations}\label{apx:limitations}
In this work, we use student models distilled from~\citet{kangaslahti2026boomerang}.
Although this saves a significant amount of compute, we inherit the limitations of the distilled models, which could potentially influence the patching order.
For example, we observe that KLPatch, out of the box, does not work well for Llama, as it significantly underperforms sequential patching order baselines.
We suspect this may be due to differences in the initialization and training of the distilled student model relative to the other distilled models.
See Appendix~\ref{apx:llama} for more details.

\section{Reproducibility Statement}\label{apx:reproducibility_statement}
We conduct all experiments using PyTorch~\citep{paszke:neurips19} and the Hugging Face Transformers library~\citep{wolf:arxiv19}.
We also use the publicly available codebase and distilled student models from \cite{kangaslahti2026boomerang} and use pretrained models from the Hugging Face Hub.

\section{Broader Impacts}\label{apx:broader_impact}
We use the distilled student models from~\citet{kangaslahti2026boomerang} and may therefore inherit biases or limitations present in these models.
Before deployment, we recommend carefully evaluating the models on the intended target tasks and conducting appropriate post-training and safety evaluations.

\section{Statement on Use of Large Language Models}\label{apx:llm_usage}
We utilized generative AI tools for code completion, debugging, and minor grammatical corrections in the manuscript. 
The authors carried out all the substantive research contributions, analyses, and interpretations. 

\section{Proofs}\label{apx:proofs}

\setcounter{proposition}{0}
\begin{proposition}[Shortest KL Paths are optimal]
The shortest paths in \(\mathcal G\) are exactly the optimal permutations with respect to \(\mathrm{AUPIC}^{\log}_\mT(\pi)\) (to be understood as a set-equality if there are multiple minimizers):
\[
\arg\min_{\pi\in\mathfrak S_N} E(\pi)
=
\arg\min_{\pi\in\mathfrak S_N}
\mathrm{AUPIC}^{\log}_{\mT}(\pi).
\]
\end{proposition}

\begin{proof}
For any intermediate model \(\mM_A\), \(A\subseteq\{1,\ldots,N\}\), the cross-entropy decomposition gives
\[
H(p_\mT(\cdot\mid x),p_{\mM_A}(\cdot\mid x))
=
H(p_\mT(\cdot\mid x))
+
\mathrm{KL}
\left(
p_\mT(\cdot\mid x)
\,\middle\|\,
p_{\mM_A}(\cdot\mid x)
\right).
\]
Averaging over \(x\sim\mathcal D_{\mathrm{cal}}\) and rearranging, we obtain
\[
K(A)
=
\log \mathrm{PPL}_{\mT}(\mM_A)
-
\log \mathrm{PPL}_{\mT}(\mT).
\]
Applying this identity to each intermediate model \(\mM_k^\pi\), i.e.\ to \(A=A_k^\pi\), and summing over the edges of \(\gamma_\pi\),
\begin{align*}
E(\pi)
&=
\sum_{k=1}^{N} d_{\pi_k}\,
\frac{K(A^\pi_{k-1})+K(A^\pi_k)}{2}\\
&=
\sum_{k=1}^{N} d_{\pi_k}\,
\frac{\log \mathrm{PPL}_{\mT}(\mM^\pi_{k-1})+\log \mathrm{PPL}_{\mT}(\mM^\pi_k)}{2}
-
\log \mathrm{PPL}_{\mT}(\mT)\sum_{k=1}^{N} d_{\pi_k}
\\
&=
\mathrm{AUPIC}^{\log}_{\mT}(\pi)
-
\left(|\mT|-|\mS|\right)\log \mathrm{PPL}_{\mT}(\mT),
\end{align*}
using the fact that \(\sum_{k=1}^{N} d_{\pi_k}=|\mM^\pi_N|-|\mM^\pi_0|=|\mT|-|\mS|\), since $\mM^\pi_0=\mS$ and $\mM^\pi_N=\mT$.
The final term is independent of \(\pi\), so \(E(\pi)\) and
\(\mathrm{AUPIC}^{\log}_{\mT}(\pi)\) have the same minimizers.
\end{proof}

\begin{proposition}[Reference Distribution]\label{prop:reference-distribution}
If
\[
\left|
\log \mathrm{PPL}_{\mathrm{data}}(\mM_k^\pi)
-
\log \mathrm{PPL}_{\mT}(\mM_k^\pi)
\right|
\leq \eta
\]
for all evaluated \(k=0,\ldots,N\), then
\[
\left|
\mathrm{AUPIC}^{\log}_{\mT}(\pi)
-
\mathrm{AUPIC}^{\log}(\pi)
\right|
\leq
\left(|\mT|-|\mS|\right)\eta,
\]
and
\[
e^{-\eta}\mathrm{AUPIC}_{\mT}(\pi)
\leq
\mathrm{AUPIC}(\pi)
\leq
e^{\eta}\mathrm{AUPIC}_{\mT}(\pi).
\]
Equality holds in all three inequalities when \(p_\mT=p_{\mathrm{data}}\) on \(\mathcal D_{\mathrm{cal}}\).
\end{proposition}

\begin{proof}
For each \(k\), define
\[
a_k
:=
\log \mathrm{PPL}_{\mathrm{data}}(\mM_k^\pi),
\qquad
b_k
:=
\log \mathrm{PPL}_{\mT}(\mM_k^\pi).
\]
By assumption, \(|a_k-b_k|\leq \eta\) for all \(k=0,\ldots,N\). Therefore,
by the triangle inequality,
\begin{align*}
\left|
\mathrm{AUPIC}^{\log}(\pi)
-
\mathrm{AUPIC}^{\log}_{\mT}(\pi)
\right|
&=
\left|
\sum_{k=1}^{N} d_{\pi_k}\,\frac{(a_{k-1}-b_{k-1})+(a_k-b_k)}{2}
\right|\\
&\leq
\sum_{k=1}^{N} d_{\pi_k}\,\frac{|a_{k-1}-b_{k-1}|+|a_k-b_k|}{2}\\
&\leq
\eta\sum_{k=1}^{N} d_{\pi_k}\\
&=
\left(|\mT|-|\mS|\right)\eta.
\end{align*}

For the raw-scale statement, the same assumption implies
\[
b_k-\eta \leq a_k \leq b_k+\eta.
\]
Exponentiating gives
\[
e^{-\eta}e^{b_k}
\leq
e^{a_k}
\leq
e^{\eta}e^{b_k}.
\]
Since
\[
e^{a_k}=\mathrm{PPL}_{\mathrm{data}}(\mM_k^\pi),
\qquad
e^{b_k}=\mathrm{PPL}_{\mT}(\mM_k^\pi),
\]
multiplying the bounds for \(k-1\) and \(k\) by \(d_{\pi_k}/2\geq 0\) and summing over \(k=1,\ldots,N\) yields
\begin{align*}
e^{-\eta}\sum_{k=1}^{N} d_{\pi_k}\,\frac{\mathrm{PPL}_{\mT}(\mM_{k-1}^\pi)+\mathrm{PPL}_{\mT}(\mM_k^\pi)}{2}
&\leq
\sum_{k=1}^{N} d_{\pi_k}\,\frac{\mathrm{PPL}_{\mathrm{data}}(\mM_{k-1}^\pi)+\mathrm{PPL}_{\mathrm{data}}(\mM_k^\pi)}{2}\\
&\leq
e^{\eta}\sum_{k=1}^{N} d_{\pi_k}\,\frac{\mathrm{PPL}_{\mT}(\mM_{k-1}^\pi)+\mathrm{PPL}_{\mT}(\mM_k^\pi)}{2}.
\end{align*}
Equivalently,
\[
e^{-\eta}\mathrm{AUPIC}_{\mT}(\pi)
\leq
\mathrm{AUPIC}(\pi)
\leq
e^{\eta}\mathrm{AUPIC}_{\mT}(\pi).
\]

If \(p_\mT=p_{\mathrm{data}}\) on
\(\mathcal D_{\mathrm{cal}}\), then
\[
\log \mathrm{PPL}_{\mathrm{data}}(\mM_k^\pi)
=
\log \mathrm{PPL}_{\mT}(\mM_k^\pi)
\]
for every \(k\), so the above bounds hold with \(\eta=0\), giving equality.
\end{proof}

\begin{proposition}[Log Scale versus raw Scale]\label{prop:log-vs-raw}
We have
\[
\left(|\mT|-|\mS|\right)
\exp\left(
\frac{1}{|\mT|-|\mS|}\mathrm{AUPIC}^{\log}(\pi)
\right)
\leq
\mathrm{AUPIC}(\pi),
\]
with equality if and only if
\(\log \mathrm{PPL}_{\mathrm{data}}(\mM_k^\pi)\) is constant in \(k\) (assuming \(d_i>0\) for all \(i\)).
The same statement is true for \(\mathrm{AUPIC}^{\log}_{\mT}\) and \(\mathrm{AUPIC}_{\mT}\) in place of
\(\mathrm{AUPIC}^{\log}\) and \(\mathrm{AUPIC}\).
\end{proposition}

\begin{proof}
Let
\[
z_k
:=
\log \mathrm{PPL}_{\mathrm{data}}(\mM_k^\pi),
\qquad
k=0,\ldots,N.
\]
Then
\[
\mathrm{AUPIC}^{\log}(\pi)
=
\sum_{k=1}^{N} d_{\pi_k}\,\frac{z_{k-1}+z_k}{2},
\qquad
\mathrm{AUPIC}(\pi)
=
\sum_{k=1}^{N} d_{\pi_k}\,\frac{e^{z_{k-1}}+e^{z_k}}{2}.
\]
Both are weighted sums over the \(2N\) terms \(z_{k-1},z_k\) (\(k=1,\ldots,N\)) with nonnegative weights \(d_{\pi_k}/2\) that add up to \(\sum_{k=1}^{N} d_{\pi_k}=|\mT|-|\mS|\).
Since the exponential function is convex, Jensen's inequality gives
\[
\exp\left(
\frac{1}{|\mT|-|\mS|}
\sum_{k=1}^{N} d_{\pi_k}\,\frac{z_{k-1}+z_k}{2}
\right)
\leq
\frac{1}{|\mT|-|\mS|}
\sum_{k=1}^{N} d_{\pi_k}\,\frac{e^{z_{k-1}}+e^{z_k}}{2}.
\]
Multiplying both sides by \(|\mT|-|\mS|\), we obtain
\[
\left(|\mT|-|\mS|\right)
\exp\left(
\frac{1}{|\mT|-|\mS|}
\mathrm{AUPIC}^{\log}(\pi)
\right)
\leq
\mathrm{AUPIC}(\pi).
\]
By strict convexity of the exponential function, equality holds if and only if all \(z_{k-1},z_k\) with \(d_{\pi_k}>0\) coincide. If \(d_i>0\) for all \(i\), this means
\[
z_0=z_1=\cdots=z_N,
\]
that is, if and only if
\(\log \mathrm{PPL}_{\mathrm{data}}(\mM_k^\pi)\) is constant in \(k\).

The same Jensen inequality and equality condition also hold with
\(\mathrm{AUPIC}^{\log}_{\mT}\) and \(\mathrm{AUPIC}_{\mT}\) in place of
\(\mathrm{AUPIC}^{\log}\) and \(\mathrm{AUPIC}\).
\end{proof}

\begin{remark}[Log vs raw AUPIC]
Log and raw AUPIC have the same monotone direction but generally distinct
argmins, and can disagree when one ordering has a single anomalously bad
intermediate model whose raw perplexity dominates the sum.
The map \(z\mapsto e^z\) is strictly increasing, so if one ordering has no larger
log-perplexity than another ordering at every interpolation point, then it also has
no larger raw perplexity at every interpolation point. In this pointwise sense, the
log and raw objectives have the same monotone direction.

However, the two summed objectives need not have the same minimizer. The log
objective sums the \(z_k\), while the raw objective sums \(e^{z_k}\), so the raw
objective penalizes isolated large values much more heavily. For example, consider
two candidate interpolation curves with the same endpoints and intermediate
log-perplexities
\[
(z_0,z_1,z_2,z_3)=(0,0,6,0),
\qquad
(z'_0,z'_1,z'_2,z'_3)=(0,3.1,3.1,0).
\]
Then
\[
\sum_{k=0}^{3} z_k = 6
<
6.2
=
\sum_{k=0}^{3} z'_k,
\]
so the first curve is preferred on the log scale. But
\[
\sum_{k=0}^{3} e^{z_k}
=
3+e^6
>
2+2e^{3.1}
=
\sum_{k=0}^{3} e^{z'_k},
\]
so the second curve is preferred on the raw scale. Thus log-scale and raw-scale
AUPIC can have distinct argmins.
\end{remark}

\section{KLPatch Update Rule}\label{apx:update-rule}

KLPatch (Algorithm~\ref{alg:greedy-patching-order}) builds a path greedily, one edge at a time. This section compares two natural greedy rules, shows that one of them is optimal when the KL gains of the layers do not interact, and gives an example where the other rule fails. We also show that in the equidistant setting, both rules reduce to the KL rule used in Algorithm~\ref{alg:greedy-patching-order}.

\paragraph{Two greedy rules.}
Let $A\subseteq\{1,\ldots,N\}$ be the set of layers patched so far. The \emph{cheapest-edge rule} follows the outgoing edge of smallest weight,
\[
i^\star\in\arg\min_{i\notin A}\; w(\mM_A\to\mM_{A\cup\{i\}})
=\arg\min_{i\notin A}\; d_i\,\frac{K(A)+K(A\cup\{i\})}{2},
\]
that is, it minimizes the next term of $\mathrm{AUPIC}^{\log}_\mT$. This rule is probably the first one that comes to mind, as it directly optimizes for the cheapest outgoing edge. The \emph{steepest-descent rule} maximizes the KL drop per added parameter instead,
\[
i^\star\in\arg\max_{i\notin A}\;\frac{K(A)-K(A\cup\{i\})}{d_i},
\]
that is, it follows the edge of steepest slope in the (size, KL) plane. If $d_i=0$, we set the ratio to $+\infty$ if $K(A)-K(A\cup\{i\})>0$ and to $-\infty$ otherwise. Note that the total width $|\mT|-|\mS|=\sum_i d_i$ of the interpolation curve is fixed, since every layer is eventually patched. Choosing a narrow step therefore saves nothing; it only postpones the wide steps. What matters is how fast the curve descends, and a lower curve now lowers every remaining trapezoid. The cheapest-edge rule only sees the next trapezoid, whose weight is dominated by the width $d_i$ whenever $K(A)$ is large compared to the drop. We make this precise under a simplifying assumption, under which we show that the steepest-descent role is optimal.

\begin{definition}[Additive gains]\label{def:additive-gains}
We say that $K$ has \emph{additive gains} if there exist $g_1,\ldots,g_N\in\mathbb R$ such that
\[
K(A)=K(\emptyset)-\sum_{i\in A} g_i
\qquad\text{for all } A\subseteq\{1,\ldots,N\}.
\]
\end{definition}
In words, patching layer $i$ lowers the KL to the teacher by $g_i$ regardless of which other layers are patched. Necessarily $g_i=K(\emptyset)-K(\{i\})$ and $\sum_i g_i=K(\emptyset)$, since $K(\{1,\ldots,N\})=0$.

\begin{proposition}[Steepest descent is optimal under additive gains]\label{prop:additive-gains}
Assume $d_i>0$ for all $i$ and that $K$ has additive gains. Then:
\begin{enumerate}[leftmargin=1cm]
\item[(i)] $E(\pi)=\sum_{k=1}^{N} g_{\pi_k}\,|\mM^\pi_k|+\mathrm{const}$, where the constant does not depend on $\pi$.
\item[(ii)] $\pi$ minimizes $E$ over $\mathfrak S_N$ if and only if
$g_{\pi_1}/d_{\pi_1}\ge g_{\pi_2}/d_{\pi_2}\ge\cdots\ge g_{\pi_N}/d_{\pi_N}$.
\item[(iii)] The steepest-descent rule returns a minimizer of $E$, and hence, by Proposition~\ref{prop:paths}, a minimizer of $\mathrm{AUPIC}^{\log}_\mT$, for every tie-breaking rule.
\end{enumerate}
\end{proposition}

\begin{proof}
\emph{(i)} Under additive gains, $K(A^\pi_k)=K(\emptyset)-\sum_{m=1}^{k} g_{\pi_m}$ for $k=0,\ldots,N$. Hence
\[
E(\pi)
=\sum_{k=1}^{N} d_{\pi_k}\,\frac{K(A^\pi_{k-1})+K(A^\pi_k)}{2}
=\left(|\mT|-|\mS|\right)K(\emptyset)
-\frac12\sum_{k=1}^{N} d_{\pi_k}\Big(\sum_{m=1}^{k-1} g_{\pi_m}+\sum_{m=1}^{k} g_{\pi_m}\Big).
\]
Exchanging the order of summation and using $\sum_{k=m+1}^{N} d_{\pi_k}=|\mT|-|\mM^\pi_m|$ and $\sum_{k=m}^{N} d_{\pi_k}=|\mT|-|\mM^\pi_m|+d_{\pi_m}$,
\[
\sum_{k=1}^{N} d_{\pi_k}\Big(\sum_{m=1}^{k-1} g_{\pi_m}+\sum_{m=1}^{k} g_{\pi_m}\Big)
=\sum_{m=1}^{N} g_{\pi_m}\left(2|\mT|-2|\mM^\pi_m|+d_{\pi_m}\right).
\]
Therefore
\[
E(\pi)=\sum_{m=1}^{N} g_{\pi_m}\,|\mM^\pi_m|
+\Big[\left(|\mT|-|\mS|\right)K(\emptyset)-|\mT|\sum_{i=1}^{N} g_i-\frac12\sum_{i=1}^{N} g_i d_i\Big],
\]
and the bracket does not depend on $\pi$.

\emph{(ii)} Let $\pi$ satisfy $\pi_k=i$ and $\pi_{k+1}=j$, and let $\pi'$ be obtained from $\pi$ by swapping $i$ and $j$. The models at all steps other than $k$ and $k+1$ are the same under $\pi$ and $\pi'$, so by (i)
\[
E(\pi)-E(\pi')
=g_i|\mM^\pi_k|+g_j|\mM^\pi_{k+1}|-g_j|\mM^{\pi'}_k|-g_i|\mM^{\pi'}_{k+1}|.
\]
Since $|\mM^\pi_k|=|\mM^\pi_{k-1}|+d_i$, $|\mM^{\pi'}_k|=|\mM^\pi_{k-1}|+d_j$ and $|\mM^\pi_{k+1}|=|\mM^{\pi'}_{k+1}|=|\mM^\pi_{k-1}|+d_i+d_j$, this gives
\[
E(\pi)-E(\pi')=g_j d_i-g_i d_j.
\]
Hence $E(\pi)>E(\pi')$ if and only if $g_i/d_i<g_j/d_j$. Now let $\pi^\ast$ be a minimizer of $E$, which exists since $\mathfrak S_N$ is finite. If two consecutive layers $\pi^\ast_k=i$, $\pi^\ast_{k+1}=j$ satisfied $g_i/d_i<g_j/d_j$, swapping them would strictly decrease $E$, a contradiction. So $\pi^\ast$ is sorted by nonincreasing ratio $g/d$. Conversely, let $\pi$ be any permutation sorted by nonincreasing ratio. Then $\pi$ and $\pi^\ast$ have the same ratio at every step, and $\pi$ can be obtained from $\pi^\ast$ by successively swapping consecutive layers with equal ratio. Each such swap changes $E$ by $g_j d_i-g_i d_j=0$, so $E(\pi)=E(\pi^\ast)$ and $\pi$ is a minimizer.

\emph{(iii)} Under additive gains, $K(A)-K(A\cup\{i\})=g_i$ for every $A$ and $i\notin A$. The steepest-descent criterion at any step is therefore $g_i/d_i$, independently of $A$, and the rule patches at each step a remaining layer with the largest ratio. Its output is sorted by nonincreasing ratio, whatever the tie-breaking, and is a minimizer of $E$ by (ii). By Proposition~\ref{prop:paths} it is also a minimizer of $\mathrm{AUPIC}^{\log}_\mT$.
\end{proof}

\paragraph{Example.}
Let $N=2$ with block sizes $d_1=1$, $d_2=2$ and additive gains $g_1=1$, $g_2=4$, so that $K(\emptyset)=5$, $K(\{1\})=4$, $K(\{2\})=1$ and $K(\{1,2\})=0$. The cheapest-edge rule compares $w(\mM_\emptyset\to\mM_{\{1\}})=1\cdot\tfrac{5+4}{2}=4.5$ with $w(\mM_\emptyset\to\mM_{\{2\}})=2\cdot\tfrac{5+1}{2}=6$ and patches layer~1 first, giving $E=4.5+2\cdot\tfrac{4+0}{2}=8.5$. The steepest-descent rule compares $g_1/d_1=1$ with $g_2/d_2=2$ and patches layer~2 first, giving $E=6+1\cdot\tfrac{1+0}{2}=6.5$, which is optimal by Proposition~\ref{prop:additive-gains}. The first edge of the steepest-descent path is more expensive, but it lowers the curve for the entire remaining size range. For these block sizes, the cheapest-edge rule patches layer~1 first whenever $g_1>0$, however small $g_1$ and however large $g_2$, whereas this is optimal only if $g_1\ge g_2/2$.

\begin{remark}[Equidistant block sizes]\label{rem:equidistant-rules}
If all blocks have the same size, $d_i=c>0$ for all $i$, then $c$ and $K(A)$ are common to all candidates $i\notin A$, and both rules reduce to
\[
i^\star\in\arg\min_{i\notin A} K(A\cup\{i\}),
\]
that is, patch the layer whose patched model has the smallest KL divergence to the teacher. In the equidistant setting the two rules therefore reduce to the KL rule used in Algorithm~\ref{alg:greedy-patching-order}.
\end{remark}

\begin{remark}[Non-additive gains in practice]\label{rem:non-additive}
Under additive gains the steepest-descent criterion does not depend on $A$, so the non-iterative variant KLInitial of Appendix~\ref{apx:klpatch-non-iterative}, which ranks the layers once by $K(\{i\})$, returns the same order as KLPatch whenever the block sizes are equal. The gap between KLPatch and KLInitial observed in Appendix~\ref{apx:klpatch-non-iterative} thus shows that the gains are not additive for the models we study, and that the iterative recomputation in Algorithm~\ref{alg:greedy-patching-order} is needed. Without additivity, no rule that looks only one step ahead is optimal in general, since the best first patch can depend on interactions that are only visible after further patching.
\end{remark}

\section{Experimental Details}\label{apx:experimental_details}

\paragraph{Models.}
We use the openly available distilled models released by ~\citet{kangaslahti2026boomerang} for Qwen, Pythia, and Llama, and the DistilBERT and DistilGPT models released by ~\citet{sanh2019distilbert} on Hugging Face. 
All the interpolated models are created without any additional training. 

\paragraph{Evaluation Datasets.} 
We use the same evaluation datasets as ~\citet{kangaslahti2026boomerang} throughout the paper. 
We report the classification accuracy on these ten tasks: ARC-easy and ARC-challenge \citep{clark2018thinksolvedquestionanswering}, BoolQ \citep{clark2019boolqexploringsurprisingdifficulty}, HellaSwag~\citep{zellers2019hellaswagmachinereallyfinish}, OpenBookQA \citep{mihaylov2018suitarmorconductelectricity}, PIQA \citep{bisk2019piqareasoningphysicalcommonsense}, WinoGrande~\citep{sakaguchi2019winogrande}, RACE \citep{lai-etal-2017-race}, MMLU \citep{hendrycks2021measuringmassivemultitasklanguage}, and RTE \citep{wang-etal-2018-glue}. 
For generation tasks, we report the exact match on the following datasets: GSM8K \citep{cobbe2021trainingverifierssolvemath}, IFEval \citep{zhou2023instructionfollowingevaluationlargelanguage}, and MATH \citep{hendrycksmath2021}. 
We report perplexity and pseudo-perplexity on the Wikitext dataset~\citep{wikitext}. 
We use \verb|lm-| \verb|evaluation-harness| \citep{eval-harness} to evaluate all of the models used in our experiments. 

\paragraph{Hardware and Walltime.}
For the DistilBERT and DistilGPT interpolation experiments, we use a single NVIDIA A100 40GB GPU, whereas for the LLM interpolation experiments, we use a single NVIDIA H100 80GB GPU to create and evaluate the intermediate models. 
Evaluating a single intermediate model in a trajectory is completed in under an hour.

\section{Empirically Validating Theoretical Assumptions}\label{apx:assumption-validation}

In this section, we empirically validate the assumptions made in Section \ref{sec:model-patching} and Appendix \ref{apx:proofs}.

\subsection{KL Divergence Strongly Correlates with Perplexity}\label{apx:kl-vs-perplexity}

In Proposition \ref{prop:paths}, we optimize for the perplexity with respect to the teacher distribution rather than the data distribution. In this section, we show that optimizing w.r.t. the teacher distribution is reasonable by demonstrating a strong correlation between the KL divergence to the teacher model and downstream perplexity.

\paragraph{Setup.} We use the exhaustive set of DistilBERT and DistilGPT2 patching orders from Section \ref{sec:experiment:distilbert} to study the relationship between KL divergence and data perplexity.
We evaluate KL divergence on the Wikitext test set ~\citep{wikitext} and on a calibration set of $64$ samples from the Pile ~\citep{gao2020pile800gbdatasetdiverse}, and evaluate perplexity on the Wikitext test set.

\paragraph{Results.} When KL divergence and perplexity are evaluated on the same dataset (Wikitext), there is a strong correlation (Pearson r =$0.981$ for DistilBERT and Pearson r=$0.955$ for DistilGPT2) between KL divergence and perplexity (Figure \ref{fig:kl-aupic-correlation}). 
This indicates that distributions of interpolated models that are similar to the distribution of the teacher are also similar to the data distribution. 
Figure \ref{fig:kl-path-aupic-correlation} further shows that KL path length is strongly correlated to AUPIC (Pearson r=$0.958$ for DistilBERT and Pearson r=$0.903$ for DistilGPT2), so shorter interpolation paths in $\mathcal{G}$ tend to have better performance in terms of AUPIC. 
These results also hold when the KL divergence is computed on a different dataset than the perplexity.
Figure \ref{fig:kl-aupic-correlation-pile} demonstrates that KL divergence on the Pile calibration set correlates with Wikitext perplexity, and Figure \ref{fig:kl-path-aupic-correlation-pile} shows that KL path length computed on the Pile calibration set is correlated with AUPIC. 

\begin{figure*}[!htb]
  \centering
  \includegraphics[width=0.85\linewidth]{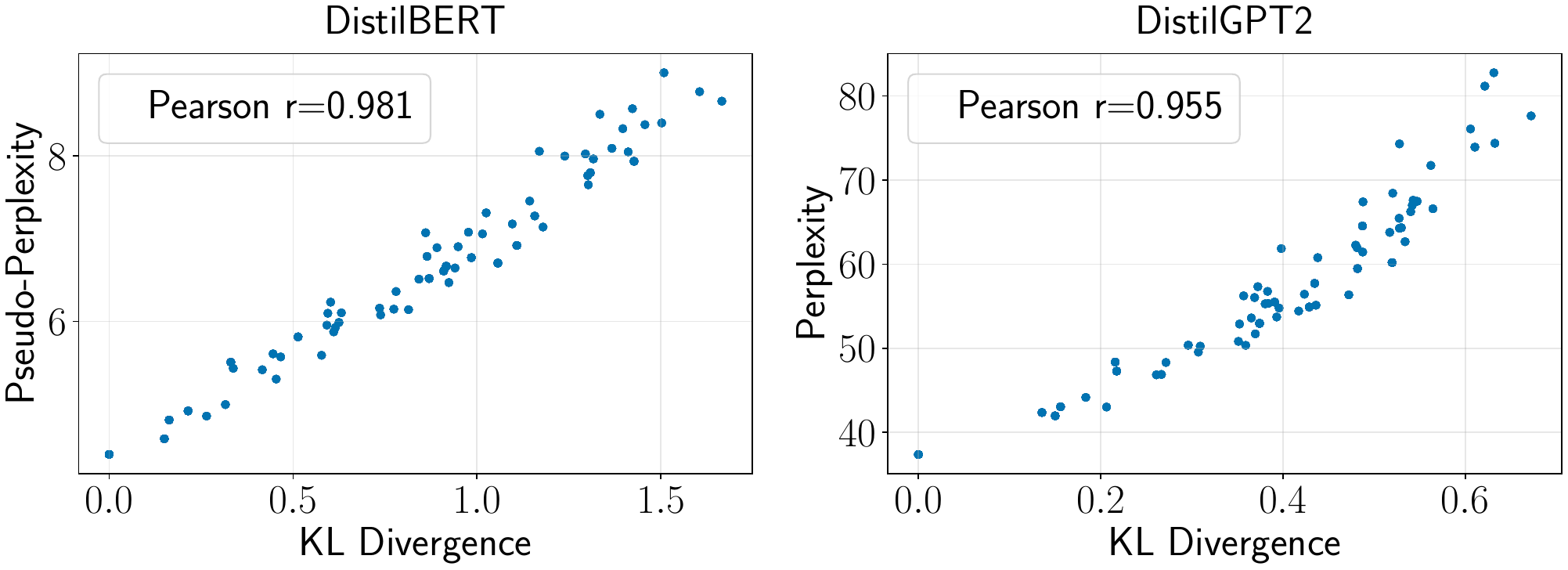}

  \caption{
  \textbf{KL divergence on Wikitext versus Wikitext perplexity.}
  When KL divergence is computed on Wikitext, there is a strong correlation between the KL divergence between interpolated models and the teacher and Wikitext perplexity for DistilBERT and DistilGPT2 models.
    }
    
  \label{fig:kl-aupic-correlation}
\end{figure*}

\begin{figure*}[!htb]
  \centering
  \includegraphics[width=0.85\linewidth]{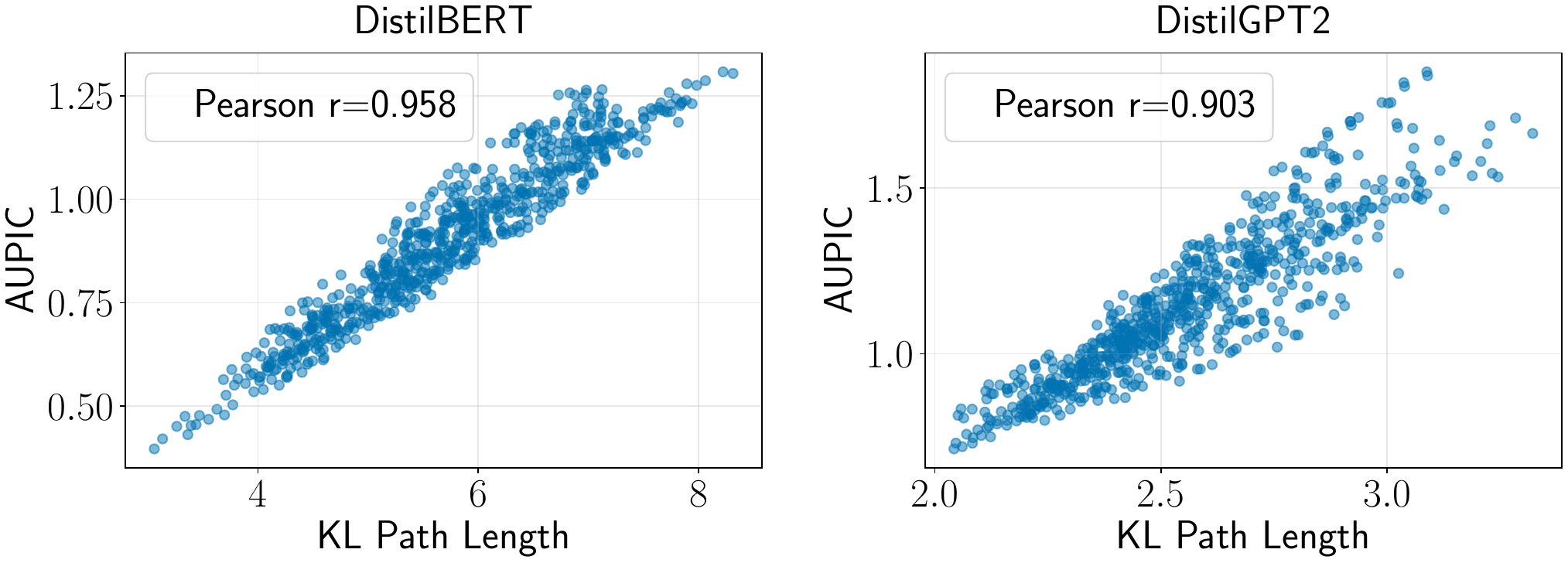}

  \caption{
  \textbf{KL Path Length on Wikitext versus Wikitext AUPIC.}
  For all possible patching orders, when KL divergence and perplexity are both computed on Wikitext, the KL path length in the interpolation graph correlates strongly with AUPIC. 
    }

  \label{fig:kl-path-aupic-correlation}
\end{figure*}

\begin{figure*}[!htb]
  \centering
  \includegraphics[width=0.85\linewidth]{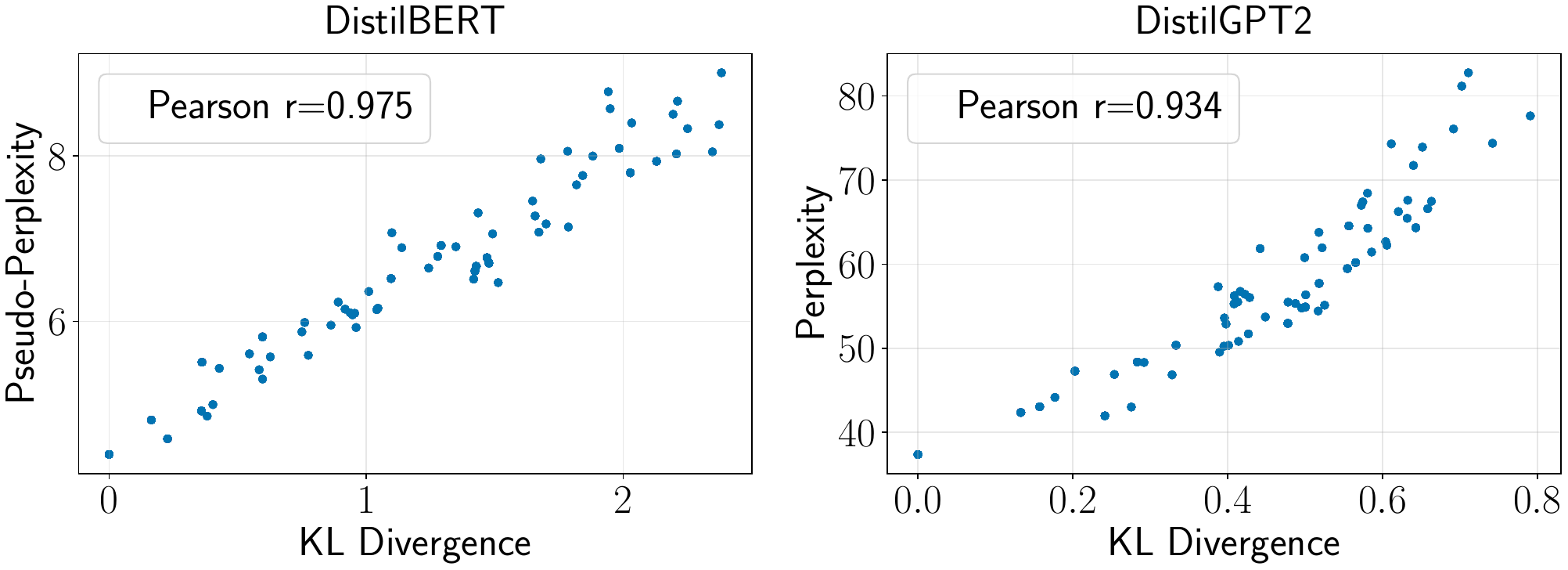}
    \vspace{-3pt}
  \caption{
  \textbf{KL divergence on the Pile versus Wikitext perplexity.}
  When KL divergence is computed on a $64$ example calibration set from the Pile, there is a strong correlation between the KL divergence between interpolated models and the teacher and Wikitext perplexity for DistilBERT and DistilGPT2 models.
    }

  \label{fig:kl-aupic-correlation-pile}
\end{figure*}

\begin{figure*}[!thb]
  \centering
  \includegraphics[width=0.85\linewidth]{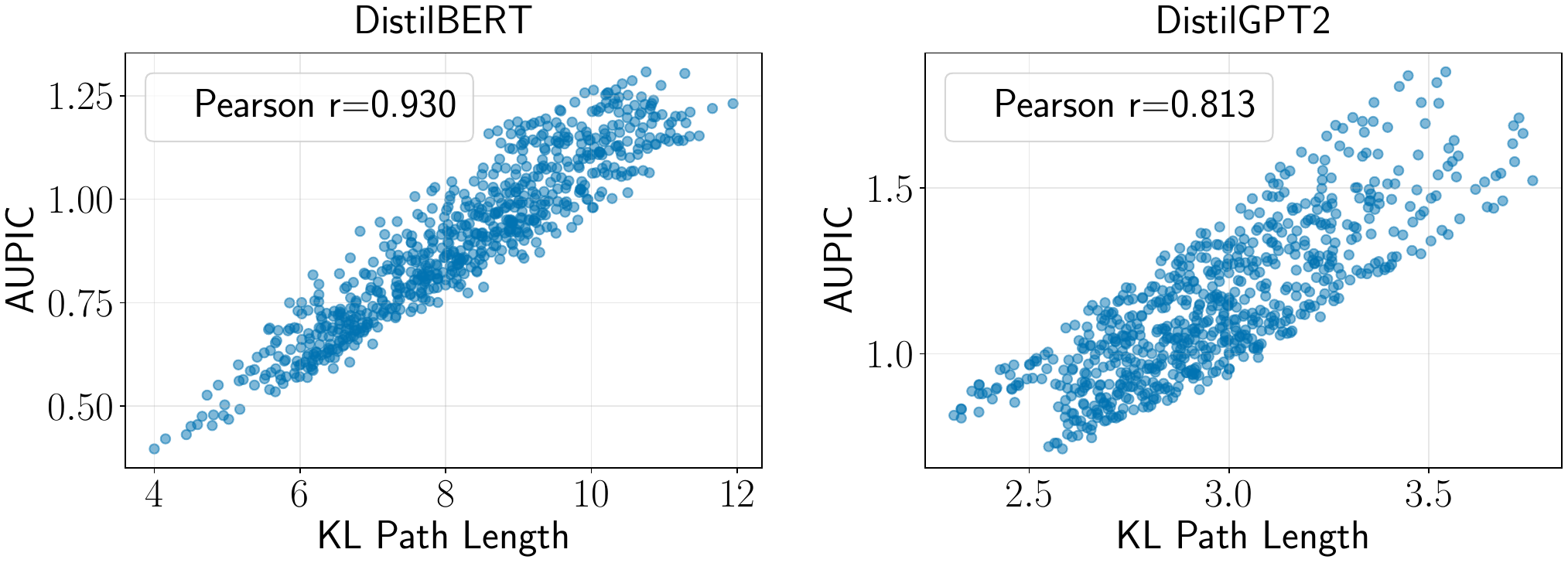}
    \vspace{-3pt}
  \caption{
  \textbf{KL Path Length on the Pile versus Wikitext AUPIC.}
  For all possible patching orders, when KL divergence is computed on a $64$ example calibration set of the Pile and perplexity is computed on Wikitext, the KL path length in the interpolation graph correlates strongly with AUPIC. 
    }

  \label{fig:kl-path-aupic-correlation-pile}
\end{figure*}

\subsection{Optimal Patching Order Closely Approximates Optimal Interpolation}\label{apx:patching-order-vs-interpolation}

In Section \ref{sec:interpolation-optimization}, we introduce the optimal interpolation problem (Problem \ref{prob:optimal-interpolation}) and the optimal permutation problem (Problem \ref{prob:optimal-permutation}) for subset selection and argue that greedy algorithms naturally solve both formulations. In this section, we empirically validate that the optimal solution to Problem \ref{prob:optimal-permutation} is close in performance to the optimal solution to Problem \ref{prob:optimal-interpolation}. Thus, restricting the search space to the set of permutations does not impose a significant cap on performance as compared to the optimal interpolation problem.

\paragraph{Setup.} We evaluate the perplexity of the exhaustive set of DistilBERT and DistilGPT2 patching orders from Section \ref{sec:experiment:distilbert} and evaluate the classification accuracy of the set of $200$ patching orders for Qwen3-4B-Base, Qwen3-8B-Base, and Pythia-6.9b from Section \ref{sec:experiment:patching_order}.

\paragraph{Results.} Figure \ref{fig:patching-vs-interpolation-distilbert-distilgpt} shows that for DistilBERT models, the optimal permutation (lower dotted blue line) is exactly equal to the optimal interpolation (lower dotted gray line) for each model size. As a result, the minimum AUPIC for the optimal interpolation and patching order are identical (Table \ref{tab:aupic_interpolations}). For DistilGPT2, the optimal permutation is slightly higher in perplexity than the optimal interpolation for small model sizes (Figure \ref{fig:patching-vs-interpolation-distilbert-distilgpt}), but is very close in AUPIC to the optimal interpolation (Table \ref{tab:aupic_interpolations}).

For larger models, while it is computationally infeasible to compute the exact optimal permutation and interpolation, Figure \ref{fig:patching-vs-interpolation-llms} demonstrates that on the subset of $200$ patching orders that we evaluated, the maximum AUIC permutation (higher blue dotted line) has a classification accuracy close to that of the maximum interpolation of corresponding size (higher gray dotted line). We observe analogous results for perplexity and AUPIC in Figure \ref{fig:patching-vs-interpolation-llms-wikitext}. Similarly, the maximum AUIC across permutations is close to the maximum AUIC interpolation (Table \ref{tab:auic_interpolations}). Overall, these results indicate that across models and tasks, the optimal permutation (Problem \ref{prob:optimal-permutation}) is close in performance to the optimal interpolation (Problem \ref{prob:optimal-interpolation}).\looseness-1

\begin{table}[!ht]
\small
\centering
\begin{tabular}{llc}
\toprule
\textbf{Model} & \textbf{Interpolation} & \textbf{AUPIC ($\downarrow$)} \\
\midrule
           & Mean AUPIC & 0.909 \\
DistilBERT & Minimum AUPIC Interpolation & 0.397 \\
           & Minimum AUPIC Permutation & 0.397 \\
\midrule
           & Mean AUPIC & 1.072 \\
DistilGPT2 & Minimum AUPIC Interpolation & 0.657 \\
           & Minimum AUPIC Permutation & 0.713\\
\bottomrule
\end{tabular}
\vspace{5pt}
\caption{\textbf{AUPIC results for best performing interpolations vs permutations.} The minimum AUPIC permutation has AUPIC equal to or close to that of the minimum AUPIC interpolation. }
\label{tab:aupic_interpolations}
\end{table}

\begin{figure*}[!ht]
  \centering
  \includegraphics[width=0.85\linewidth]{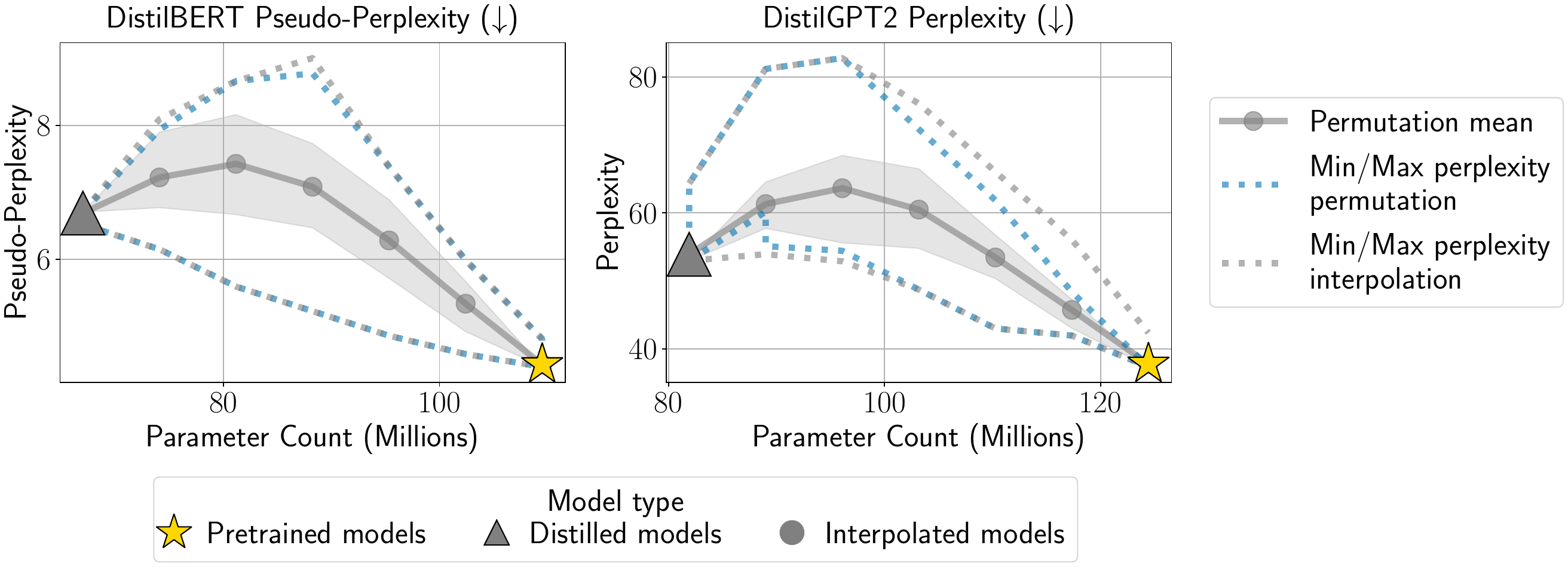}

  \caption{
  \textbf{Comparison of minimum AUPIC permutation vs interpolation for DistilBERT and DistilGPT2.} The minimum permutation produces the same perplexity as the minimum interpolation for DistilBERT and perplexity close to the minimum interpolation for DistilGPT2. Shaded bands show the (25-75) inter-quartile range over all 720 orderings.
  }

  \label{fig:patching-vs-interpolation-distilbert-distilgpt}
\end{figure*}

\begin{table}[!ht]
\centering
\small
\begin{tabular}{llc}
\toprule
\textbf{Model} & \textbf{Interpolation} & \textbf{AUIC ($\uparrow$)} \\
\midrule
            & Mean AUIC & 0.409 \\
Qwen3-4B-Base  & Maximum AUIC Interpolation & 0.571 \\
           & Maximum AUIC Permutation & 0.522 \\
\midrule
          & Mean AUIC & 0.447 \\
 Qwen3-8B-Base  & Maximum AUIC Interpolation & 0.634 \\
           & Maximum AUIC Permutation & 0.585 \\
\midrule
            & Mean AUIC & 0.209 \\
 Pythia-6.9b & Maximum AUIC Interpolation & 0.459 \\
           & Maximum AUIC Permutation & 0.409 \\
\bottomrule
\end{tabular}
\vspace{5pt}
\caption{\textbf{AUIC results for best performing interpolations vs permutations.} The maximum AUIC permutation has AUIC close to that of the maximum AUIC interpolation.}
\label{tab:auic_interpolations}
\end{table}

\begin{figure*}[!ht]
  \centering
  \includegraphics[width=\linewidth]{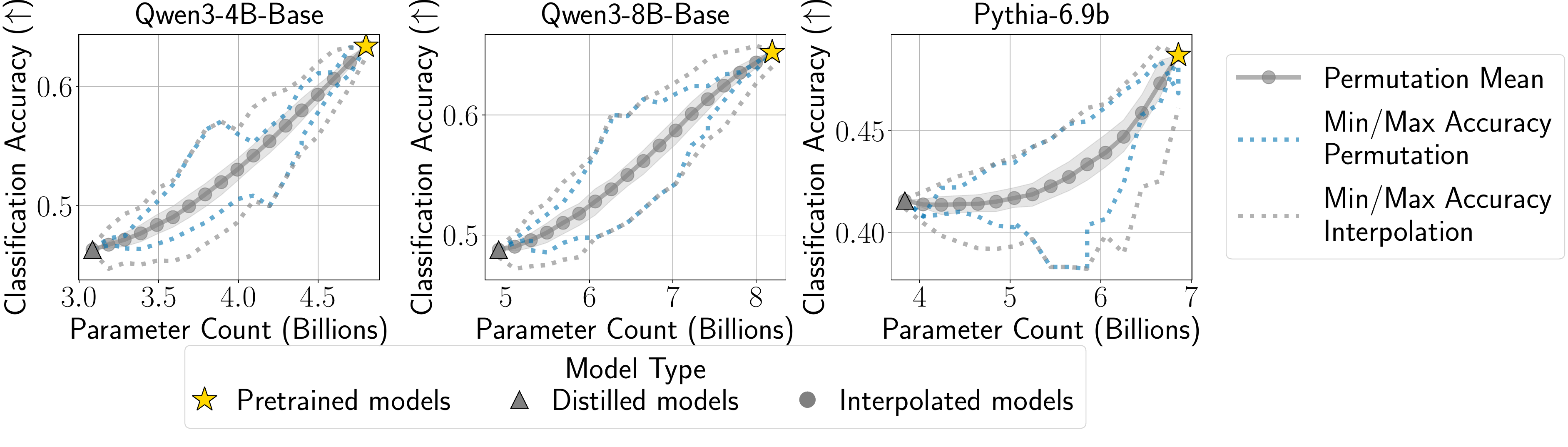}

  \caption{
  \textbf{Comparison of maximum AUIC permutation vs interpolation for Qwen3-4B-Base, Qwen3-8B-Base, and Pythia-6.9b over $\mathbf{200}$ random orderings.} Across models, the maximum AUIC permutation has classification accuracy very close to that of the maximum interpolation. Shaded bands show the (25-75) inter-quartile range over all $200$ orderings.
  }

  \label{fig:patching-vs-interpolation-llms}
\end{figure*}
\clearpage

\begin{figure*}[!ht]
  \centering
  \includegraphics[width=\linewidth]{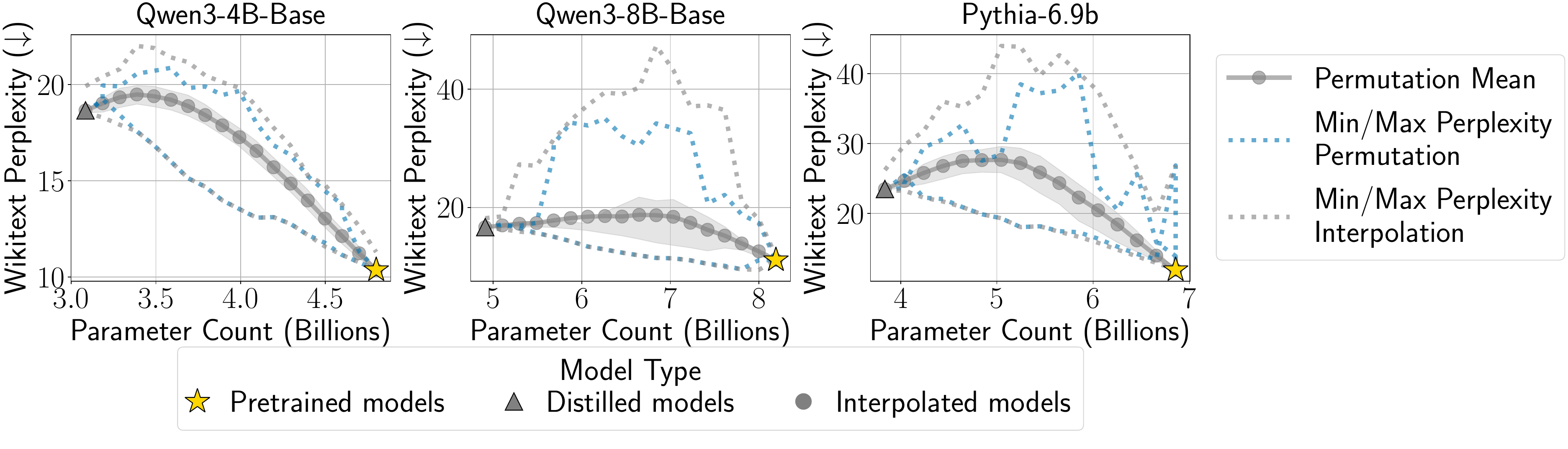}

  \caption{
  \textbf{Comparison of minimum AUPIC permutation vs interpolation for Qwen3-4B-Base, Qwen3-8B-Base, and Pythia-6.9b over $\mathbf{200}$ random orderings.} Across models, the minimum AUPIC permutation has perplexity very close to that of the minimum interpolation. Shaded bands show the (25-75) inter-quartile range over all $200$ orderings.
  }

  \label{fig:patching-vs-interpolation-llms-wikitext}
\end{figure*}

\subsection{KLPatch Path is Close to Global Shortest Path}\label{apx:klpatch-vs-global}

We prove in Proposition \ref{prop:paths} that shortest paths in \(\mathcal G\) are optimal permutations and introduce KLPatch in Section \ref{sec:klpatch} to find the approximate shortest path in a greedy manner. Here, we demonstrate that KLPatch produces paths that are close to the global shortest path in terms of AUPIC.

\paragraph{Setup.} We use the set of exhaustive patching orders for DistilBERT and DistilGPT2 models from Section \ref{sec:experiment:distilbert}. 
We compute the KL divergence for every path using a calibration set of $64$ examples from the Pile and report AUPIC for the global shortest path and the path found using our KLPatch algorithm. 
We provide the first to last and last to first patching orders as baselines.
For each patching order, we report the AUPIC and the \emph{AUPIC Percentile}, or the percent of total patching orders that have AUPIC greater than or equal to that of the patching order being tested: 
$$\text{AUPIC Percentile}(\pi) = \frac{1}{\vert \mathfrak{S}_N \vert}\sum_{\pi_i\in\mathfrak{S}_N} \mathbf{1}\{\text{AUPIC}(\pi_i) \geq \text{AUPIC}(\pi)\} $$

\paragraph{Results.} Table \ref{tab:aupic_paths_pile} shows that KLPatch produces paths with interpolation performance equal to or very close to that of the global shortest KL path. Furthermore, last to first patching provides a very strong baseline for DistilBERT and DistilGPT2 models, but KLPatch recovers an AUPIC equal to (DistilBERT) or very close to (DistilGPT2) that of the last-to-first patching. These results validate that KLPatch finds a path equal to or close in performance to that of the global shortest KL path.

\begin{table}[!ht]
\centering
\resizebox{\linewidth}{!}{%
\begin{tabular}{lllc c}
\toprule
\textbf{Model} & \textbf{Path} & \textbf{Patching Order $\pi$} & \textbf{AUPIC ($\downarrow$)} & \textbf{AUPIC Percentile ($\uparrow$)} \\
\midrule
           & Minimum AUPIC Path & $[5, 4, 3, 2, 1, 0]$ & 0.397 & 100\% \\
           & First to Last Patching & $[0, 1, 2, 3, 4, 5]$ & 0.968 & 38.75\% \\
DistilBERT & Last to First Patching & $[5, 4, 3, 2, 1, 0]$ & 0.397 & 100\% \\           
           & Global Shortest Path & $[5, 4, 3, 2, 1, 0]$ & 0.397 & 100\% \\
           & KLPatch Path & $[5, 4, 3, 2, 1, 0]$ & 0.397 & 100\% \\
\midrule
           & Minimum AUPIC Path & $[4, 5, 3, 2, 0, 1]$ & 0.713 & 100\% \\
           & First to Last Patching & $[0, 1, 2, 3, 4, 5]$ & 0.824 & 95.56\% \\
DistilGPT2 & Last to First Patching & $[5, 4, 3, 2, 1, 0]$ & 0.815 & 96.39\% \\   
           & Global Shortest KL Path & $[0, 1, 2, 3, 4, 5]$ & 0.824 & 95.56\% \\
           & KLPatch Path & $[2, 3, 0, 1, 4, 5]$ & 0.831 & 95.28\% \\
\bottomrule
\end{tabular}}
\vspace{5pt}
\caption{\textbf{AUPIC results for different patching paths.} The KLPatch path has an AUPIC close to that of the global shortest KL path.
\label{tab:aupic_paths_pile}}
\end{table}

In Figure~\ref{fig:patch-order-visualization}, we provide a visual representation of different patching orders, including the one found by KLPatch, to illustrate that KLPatch finds a permutation close to the optimal one.

\begin{figure}[!htb]
    \centering
    \hspace{.05\linewidth}\includegraphics[width=.9\linewidth]{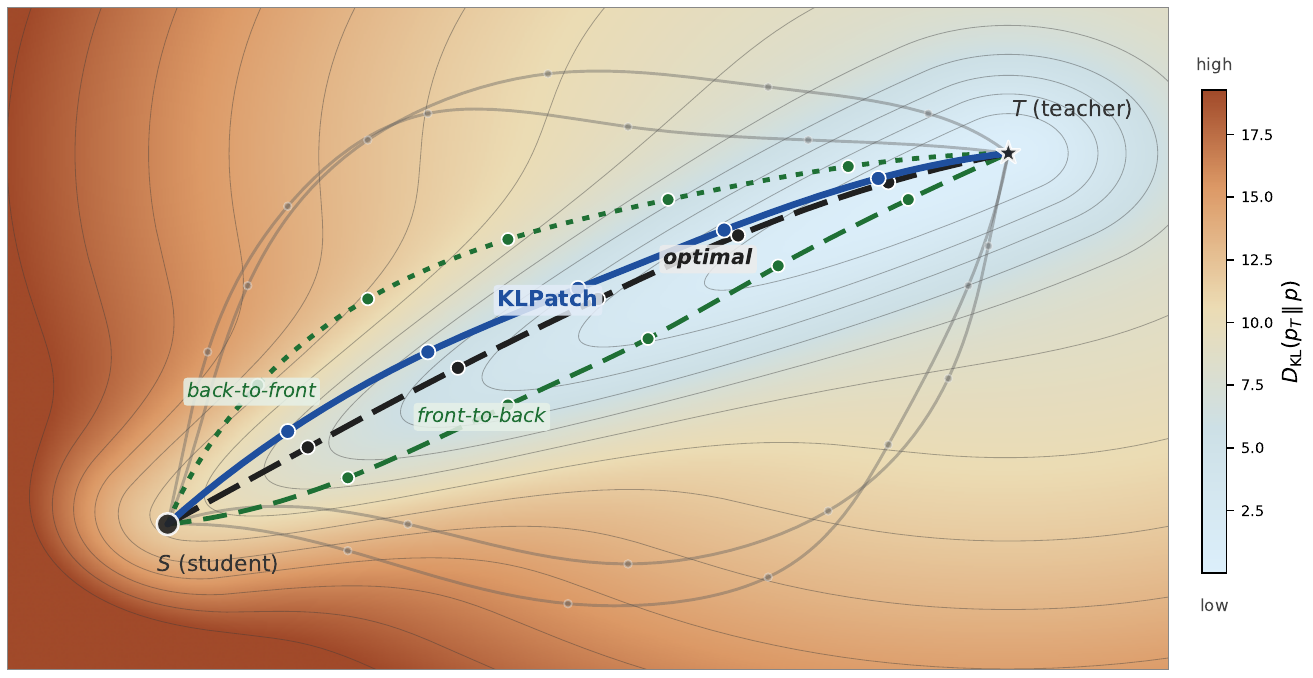}
    \caption{\textbf{KLPatch finds patching orders that are close to optimal:} Visual representation of the optimal patching order between the student $\mS$ and teacher $\mT$ along the KL divergence to the teacher model, compared to the KLPatch order, front-to-back and back-to-front, and random permutations.}
    \label{fig:patch-order-visualization}
\end{figure}

\section{Additional Permutation Results}\label{apx:additional-permutation}

We present extended results from Sections \ref{sec:experiment:distilbert} and \ref{sec:experiment:patching_order} on permutations of patching orders for model size interpolation.
Figures \ref{fig:aupic-distilbert-distilgpt}-\ref{fig:permutation-comparison-wikitext-llms} show that first to last and last to first patching orders often achieve the best interpolation trajectories compared to the random patching orders. 
These results show that simple sequential patching orders are a reliable recipe for achieving smooth model size interpolation. 

\begin{figure*}[!ht]
  \centering
  \includegraphics[width=0.8\linewidth]{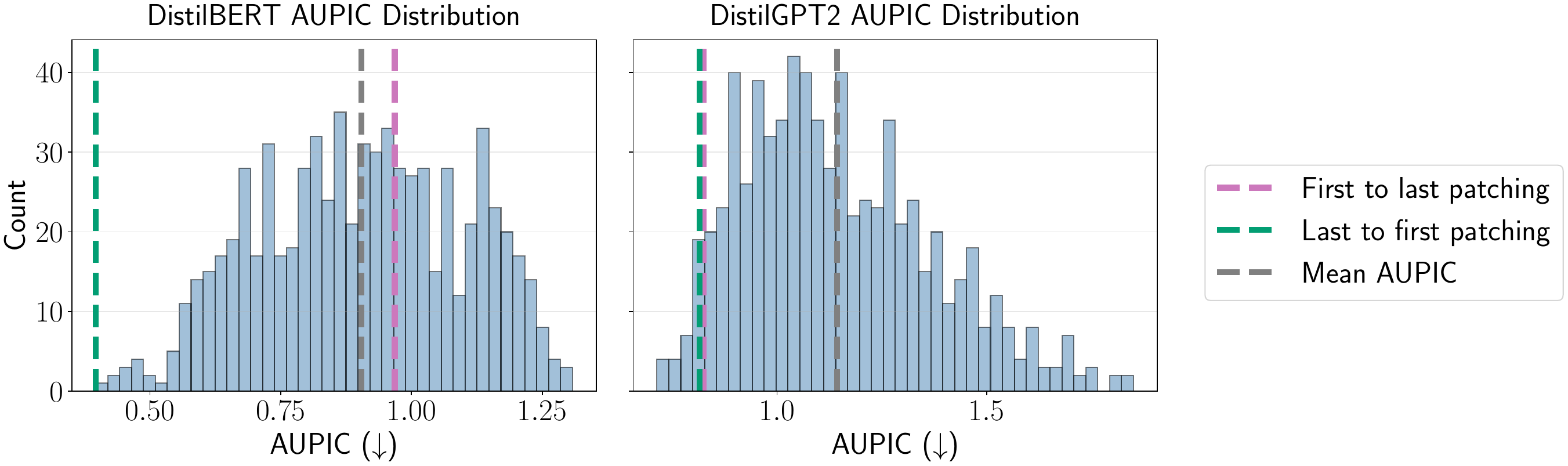}
    \caption{\textbf{AUPIC distribution across patching order permutations for DistilBERT and DistilGPT2.} Across the $720$ orderings, last-to-first patching is at or near the optimum on both models, suggesting
 it is a simple and reliable default recipe for model size interpolation.}
  \label{fig:aupic-distilbert-distilgpt}
\end{figure*}

\begin{figure*}[!ht]
  \centering
  \includegraphics[width=\linewidth]{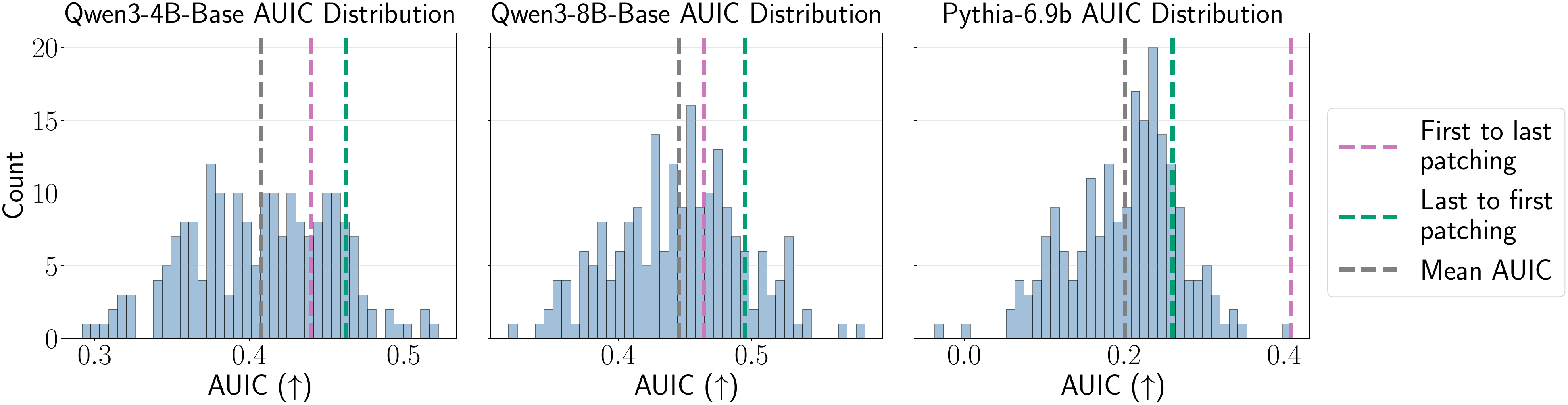}
  \caption{
  \textbf{Distribution of AUIC for different patching orders.} Across $200$ orderings, first to last and last to first patching have relatively high AUIC, but last-to-first patching does not have the highest AUIC for Qwen models despite having the lowest AUPIC.
  }
  \label{fig:permutation-comparison-llms-distribution}
\end{figure*}

\begin{figure*}[!ht]
  \centering
  \includegraphics[width=\linewidth]{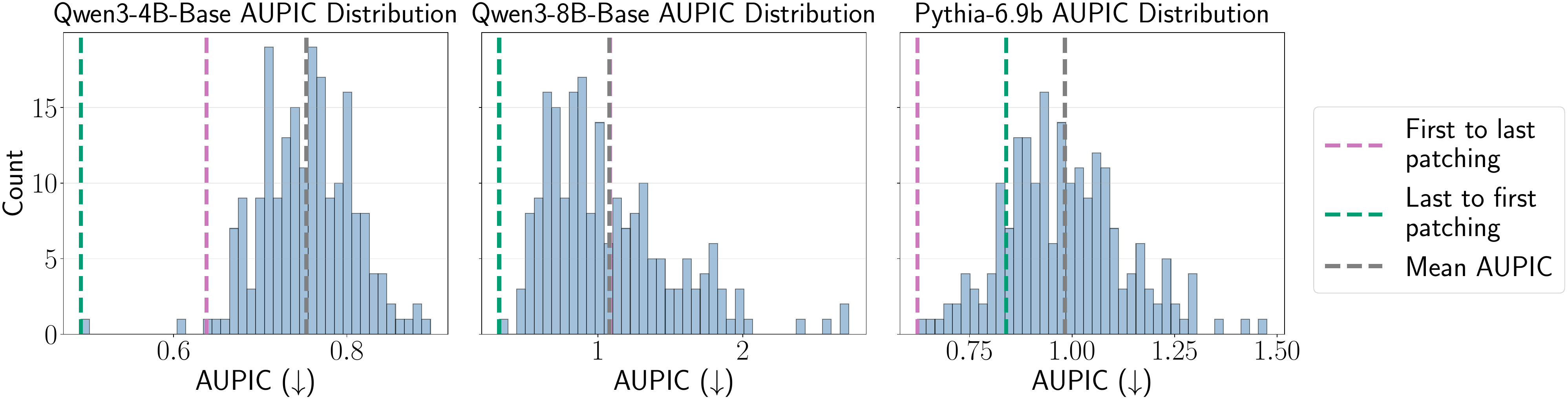}
  \caption{
  \textbf{Distribution of AUPIC for different patching orders.}
  Across the $200$ orderings, either last to first or first to last patching often yields the highest AUPIC, suggesting it is a simple and reliable default recipe for model size interpolation.
  }
  \label{fig:permutation-comparison-aupic-llms-distribution}
\end{figure*}

\begin{figure*}[!ht]
  \centering
  \includegraphics[width=\linewidth]{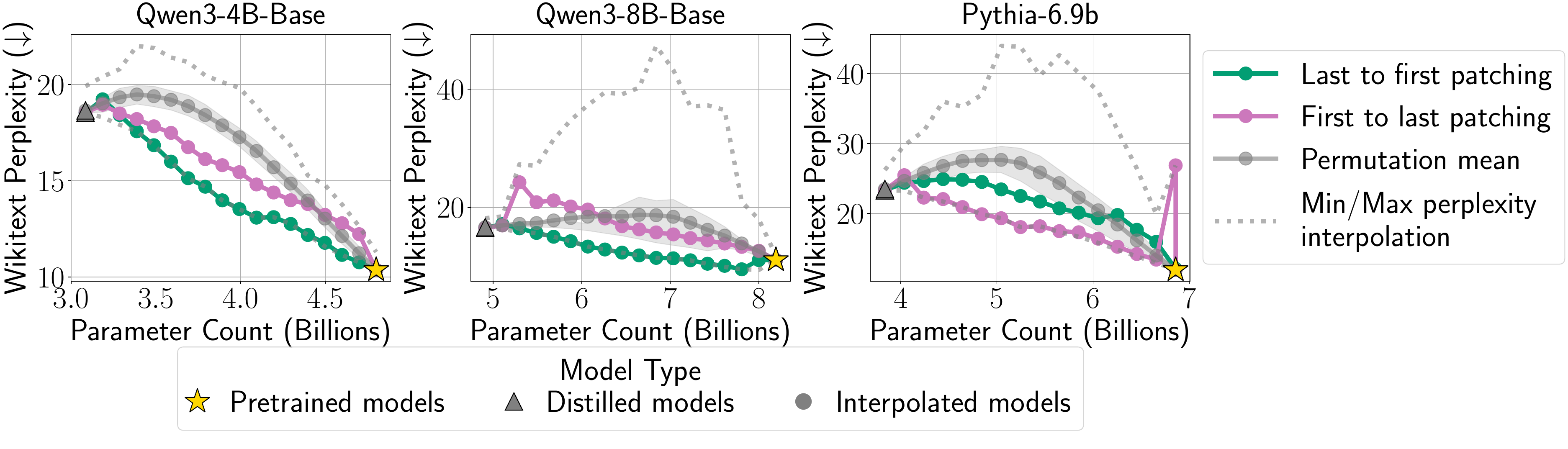}
  \caption{
  \textbf{Wikitext perplexity for different patching orders.}
  Across 200 orderings, either last to first patching and first to last patching achieve smooth interpolation performance, suggesting that they are generalizable recipes that work well across models. 
  On the other hand, we observe that a naive patching order can sometimes result in poor interpolation performance.
  Shaded bands show the (25-75) inter-quartile range over all $200$ orderings.
  }
  \label{fig:permutation-comparison-wikitext-llms}
\end{figure*}

\newpage

\section{Additional KLPatch Results}\label{apx:klpatch-additional}

We present additional results for KLPatch compared to baselines, estimated runtimes, and performance on generation tasks and Wikitext.

\subsection{Additional Baselines}\label{apx:klpatch-original-boomerang-baselines}

\paragraph{Setup.} We reproduce results from~\cite{kangaslahti2026boomerang} to provide additional baselines for boomerang distillation and better contextualize the performance of KLPatch. We compare our results to naive layer pruning, where we iteratively remove every other layer from the teacher model, starting from layer 2. This corresponds to intermediate-sized models with the same initialization as those from last to first patching, but without distilling the student model. We also provide a random initialization baseline, where the student is randomly initialized, then distilled with the same training setup and budget as in~\cite{kangaslahti2026boomerang}. This baseline shows the performance of interpolation from a standard distilled student model.

\paragraph{Results.} Figure \ref{fig:klpatch-original-boomerang-baselines} shows that all of the models constructed with boomerang distillation significantly outperform the two baselines, confirming that model size interpolation yields strong gains over standard pruning and distillation approaches. As KLPatch improves performance further compared to naive sequential patching strategies, it provides further gains in the model size interpolation setting.

\begin{figure*}[!ht]
  \centering
  \includegraphics[width=\linewidth]{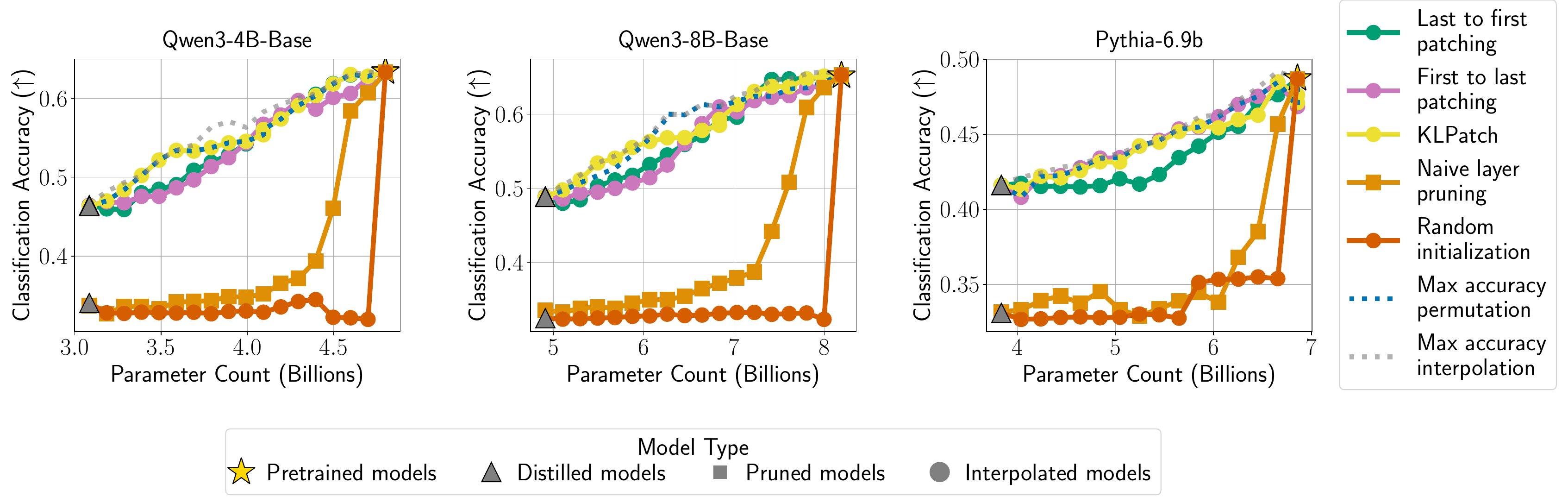}
  \caption{
  \textbf{Model size interpolation significantly outperforms layer pruning and interpolation from a standard distilled model.}
  KLPatch produces further gains over naive last to first and first to last patching in boomerang distillation.
    }

  \label{fig:klpatch-original-boomerang-baselines}
\end{figure*}

\subsection{KLPatch Trajectories are Best- or Near-Best-Performing}\label{apx:klpatch-baselines}

\paragraph{Setup.} We compute the AUIC and AUPIC for the $200$ random orderings from Section \ref{sec:experiment:patching_order}, as well as first to last, last to first, and KLPatch patching orders. We  report the \emph{AUIC percentile}, or the percent of total patching orders in the subset $S \subseteq \mathfrak{S}_N$ of orderings tested that have AUIC less than or equal to that of a given patching order $\pi$: 
$$\text{AUIC Percentile}(\pi) = \frac{1}{\vert S_N \vert}\sum_{\pi_i\in S_N} \mathbf{1}\{\text{AUIC}(\pi_i) \leq \text{AUIC}(\pi)\} $$
We also compute the AUPIC percentile (defined in Section \ref{apx:klpatch-vs-global}) over $S$ for each of the trajectories $\pi$ being considered. We note that the percentile is with respect to $203$ patching total patching orders (the set of $200$ random orderings, first to last, last to first, and KLPatch).

\paragraph{Results.} KLPatch produces the best or near-best trajectories in terms of AUIC and AUPIC across models (Tables \ref{tab:klpatch_baselines_auic}-\ref{tab:klpatch_baselines_aupic}). This shows that KLPatch is a reliable method for finding patching orders with high downstream performance.
\begin{table}[!ht]
\centering
\resizebox{0.7\linewidth}{!}{%
\small
\begin{tabular}{llcc}
\toprule
\textbf{Model} & \textbf{Interpolation} & \textbf{AUIC ($\uparrow$)} & \textbf{AUIC Percentile ($\uparrow$)} \\
\midrule
              & Maximum AUIC Interpolation & 0.580 & 100\% \\
              & Maximum AUIC Permutation & 0.533 & 100\% \\
Qwen3-4B-Base & First to Last Patching & 0.441 & 69.46\% \\
              & Last to First Patching & 0.462 & 89.16\% \\        & KLPatch Path & 0.533 & 100\% \\
\midrule
              & Maximum AUIC Interpolation & 0.637 & 100\% \\
              & Maximum AUIC Permutation & 0.585 & 100\% \\
Qwen3-8B-Base & First to Last Patching & 0.464 & 63.55\% \\
              & Last to First Patching & 0.495 & 85.71\% \\        & KLPatch Path & 0.571 & 99.51\% \\
\midrule
              & Maximum AUIC Interpolation & 0.459 & 100\% \\
              & Maximum AUIC Permutation & 0.410 & 100\% \\
Pythia-6.9b   & First to Last Patching & 0.410 & 100\% \\
              & Last to First Patching & 0.261 & 82.27\% \\        & KLPatch Path & 0.379 & 99.51\% \\
\bottomrule
\end{tabular}}
\vspace{5pt}
\caption{\textbf{AUIC results for different interpolations.} KLPatch recovers a patching order with best (Qwen3-4B-Base) or near-best (Qwen3-8B-Base, Pythia-6.9b) performance in terms of AUIC.
\label{tab:klpatch_baselines_auic}}
\vspace{-0.5cm}
\end{table}

\begin{table}[!ht]
\centering
\small
\resizebox{0.7\linewidth}{!}{%
\begin{tabular}{llcc}
\toprule
\textbf{Model} & \textbf{Interpolation} & \textbf{AUPIC ($\downarrow$)} & \textbf{AUPIC Percentile ($\uparrow$)} \\
\midrule
              & Minimum AUPIC Interpolation & 0.481 & 100\% \\
              & Minimum AUPIC Permutation & 0.493 & 100\% \\
Qwen3-4B-Base & First to Last Patching & 0.637 & 69.46\% \\
              & Last to First Patching & 0.493 & 100\% \\        & KLPatch Path & 0.539 & 99.51\% \\
\midrule
              & Minimum AUPIC Interpolation & 0.277 & 100\% \\
              & Minimum AUPIC Permutation & 0.318 & 100\% \\
Qwen3-8B-Base & First to Last Patching & 1.084 & 41.38\% \\
              & Last to First Patching & 0.318 & 100\% \\        & KLPatch Path & 0.357 & 99.51\% \\
\midrule
              & Minimum AUPIC Interpolation & 0.545 & 100\% \\
              & Minimum AUPIC Permutation & 0.567 & 100\% \\
Pythia-6.9b   & First to Last Patching & 0.623 & 99.51\% \\
              & Last to First Patching & 0.839 & 83.74\% \\        & KLPatch Path & 0.567 & 100\% \\
\bottomrule
\end{tabular}}
\vspace{5pt}
\caption{\textbf{AUPIC results for different interpolations.} KLPatch recovers a patching order with the best (Pythia-6.9b) or near-best (Qwen3-4B-Base, Qwen3-8B-Base) performance out of the permutations tested in terms of AUPIC.
\label{tab:klpatch_baselines_aupic}}
\end{table}

\subsection{KLPatch Significantly Reduces Runtime}\label{apx:klpatch-runtime}

We compare the runtime required for KLPatch $(O(N^2))$ to that of exhaustive search $(O(2^N))$ for each of the three models tested. We estimate exhaustive search runtime by recording the time it takes to compute KL divergence to the teacher on the calibration set for each intermediate model size and  multiplying it by the number of evaluations needed at that model size. As shown in Table \ref{tab:klpatch-runtime-comparison}, KLPatch achieves over $548\times$ speedup over exhaustive search for all three models, and the speedup increases as the number of student layers increases (from $17$ student layers in the Pythia model to $19$ in both Qwen models). Furthermore, exhaustive search takes a computationally infeasible amount of time as the number of student layers increases.

\begin{table}[!ht]
\centering
\small
\resizebox{\linewidth}{!}{%
\begin{tabular}{lccc}
\toprule
\textbf{Model} & \textbf{Exhaustive Search Runtime (minutes)} & \textbf{KLPatch Runtime (minutes)} & \textbf{Speedup From KLPatch} \\
\midrule
Qwen3-4B-Base & $64302.079$ & $37.337$ & $1722.209\times$ \\
Qwen3-8B-Base & $65327.869$ & $38.097$ & $1714.785\times$ \\
Pythia-6.9b & $14134.826$ & $25.760$ & $548.704\times$ \\
\bottomrule
\end{tabular}}
\vspace{0.2cm}
\caption{\textbf{KLPatch produces over $500-1700\times$ speedup over exhaustive search.} We estimate exhaustive search runtime by computing a single KL divergence evaluation pass on each intermediate model size and multiplying it by the number of evaluations needed at that model size.}
\label{tab:klpatch-runtime-comparison}
\vspace{-0.5cm}
\end{table}

\subsection{Generation Tasks}\label{apx:klpatch-generation}
Figure \ref{fig:klpatch-generation} shows that KLPatch closely tracks the best performing permutation and achieves better model size interpolation compared to last to first and first to last patching when we patch fewer student layers, but eventually converges to the best performing patching order trajectory on the generation tasks. 

\begin{figure*}[!ht]
  \centering
  \includegraphics[width=0.9\linewidth]{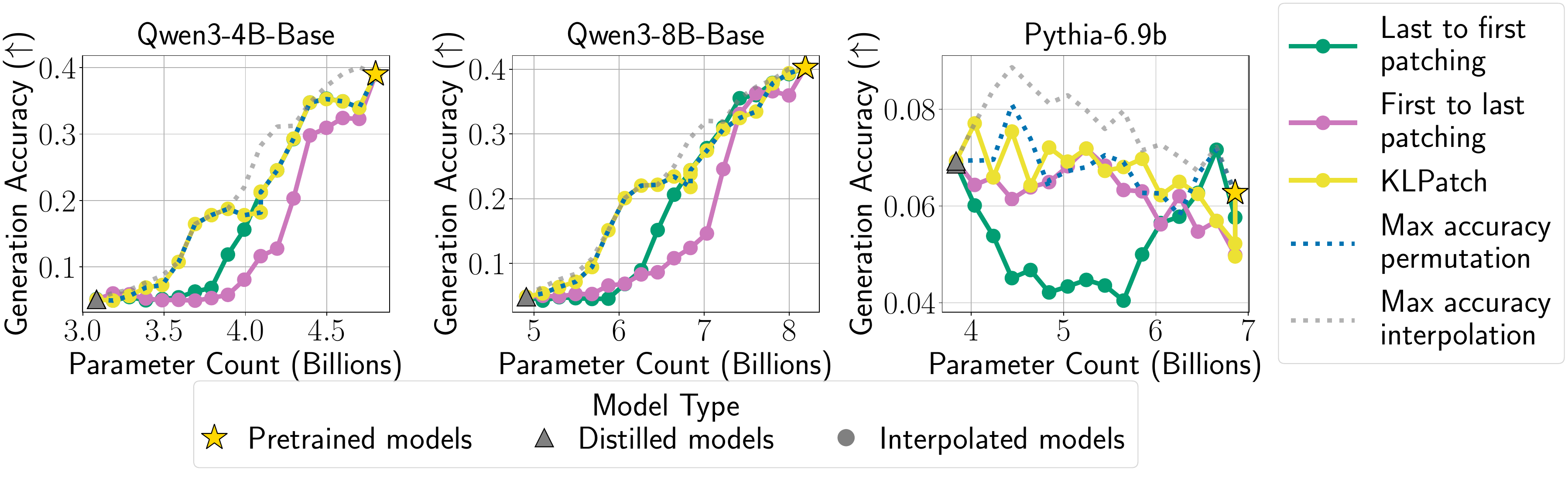}
  \caption{
  \textbf{KLPatch for Qwen3-4B-Base, Qwen3-8B-Base, and Pythia-6.9b}
  Across models, KLPatch tracks the max accuracy permutation and outperforms or tracks the last to first patching on generation tasks.}
\vspace{-0.3cm}
  \label{fig:klpatch-generation}
\end{figure*}

\subsection{Wikitext Perplexity}\label{apx:klpatch-perplexity}

Figure \ref{fig:klpatch-wikitext} shows that KLPatch often tracks closely to the best performing patching order among last to first and first to last patching orders on Wikitext. 
\clearpage
\begin{figure*}[!ht]
  \centering
  \includegraphics[width=0.9\linewidth]{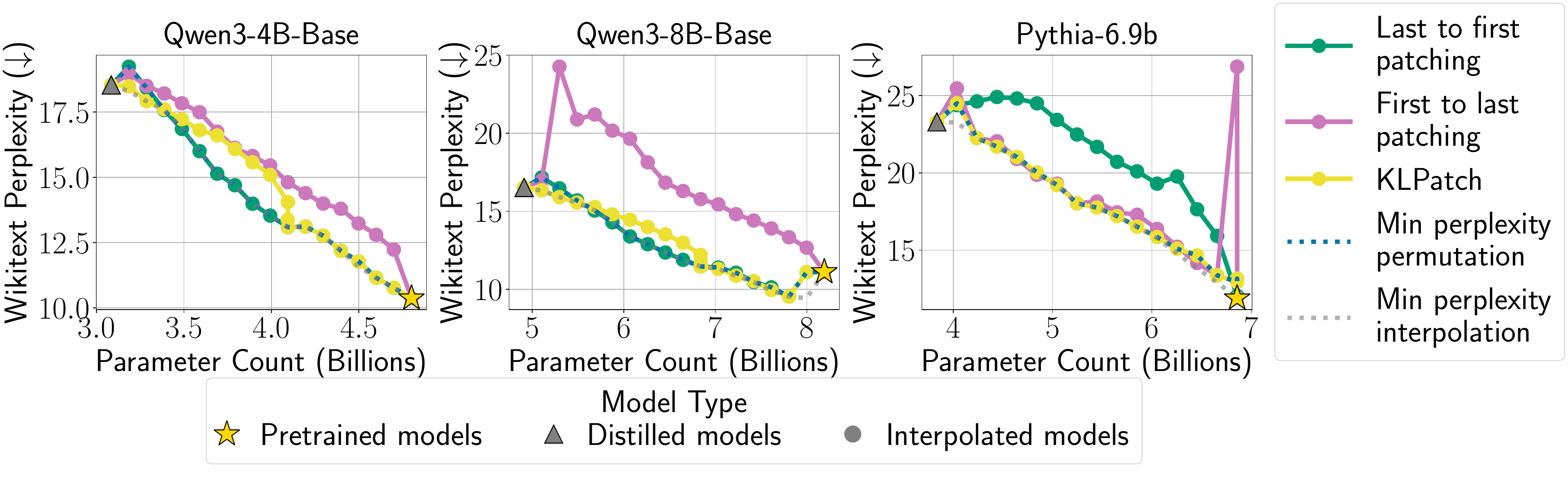}

  \caption{
  \textbf{KLPatch for Qwen3-4B-Base, Qwen3-8B-Base, and Pythia-6.9b}
  Across models, KLPatch tracks the last to first patching on perplexity.
    }

  \label{fig:klpatch-wikitext}
\end{figure*}

\subsection{Per-Task Performance}

We report a per-task breakdown of the performance of KLPatch compared to last-to-first and first-to-last patching in Tables~\ref{tab:classification-auic} and~\ref{tab:generation-auic}. We find that KLPatch is the most reliable across classification and generation results and achieves best or second-best performance on all tasks for all models other than MMLU for Pythia-6.9b.

\begin{table*}[bth]
    \centering
    \huge
    \resizebox{\textwidth}{!}{%
    \begin{tabular}{llcccccccccc}
        \toprule
        \textbf{Model} & \textbf{Method} & \textbf{ARC-challenge} & \textbf{ARC-easy} & \textbf{BoolQ} & \textbf{HellaSwag} & \textbf{MMLU} & \textbf{OpenBookQA} & \textbf{PIQA} & \textbf{RACE} & \textbf{RTE} & \textbf{WinoGrande} \\
        \midrule
        Qwen3-4B-Base
        & Last to first patching & 0.577 & 0.618 & 0.441 & 0.457 & 0.444 & 0.381 & 0.472 & 0.303 & 0.469 & 0.351 \\
        & First to last patching & 0.388 & 0.338 & 0.446 & 0.433 & 0.456 & 0.276 & 0.400 & 0.273 & 0.593 & 0.526 \\
        & KLPatch & 0.599 & 0.543 & 0.618 & 0.436 & 0.544 & 0.308 & 0.439 & 0.478 & 0.601 & 0.466 \\
        \midrule
        Qwen3-8B-Base
        & Last to first patching & 0.636 & 0.650 & 0.454 & 0.476 & 0.418 & 0.592 & 0.570 & 0.487 & 0.425 & 0.451 \\
        & First to last patching & 0.384 & 0.405 & 0.447 & 0.435 & 0.525 & 0.244 & 0.483 & 0.277 & 0.572 & 0.501 \\
        & KLPatch & 0.661 & 0.613 & 0.690 & 0.455 & 0.549 & 0.422 & 0.484 & 0.520 & 0.601 & 0.584 \\
        \midrule
        Pythia-6.9b
        & Last to first patching & 0.310 & 0.279 & -0.880 & 0.340 & 0.602 & 0.271 & 0.281 & 0.308 & 0.279 & 0.159 \\
        & First to last patching & 0.379 & 0.317 & 1.031 & 0.369 & 0.616 & 0.242 & 0.357 & 0.386 & 0.542 & 0.596 \\
        & KLPatch & 0.356 & 0.324 & 1.105 & 0.348 & 0.276 & 0.288 & 0.346 & 0.310 & 0.358 & 0.540 \\
        \bottomrule
    \end{tabular}%
    }
    \caption{\textbf{Per-task classification AUIC.}  KLPatch has the most consistent performance across tasks in terms of classification AUIC.}
    \label{tab:classification-auic}
\end{table*}

\begin{table}[ht]
    \large
    \centering
    \resizebox{0.5\textwidth}{!}{%
    \begin{tabular}{llccc}
        \toprule
        \textbf{Model} & \textbf{Method} & \textbf{GSM8K} & \textbf{MATH} & \textbf{IFEval} \\
        \midrule
        Qwen3-4B-Base
        & Last to first patching & 0.437 & 0.408 & 0.191 \\
        & First to last patching & 0.259 & 0.173 & 0.316 \\
        & KLPatch & 0.520 & 0.386 & 0.213 \\
        \midrule
        Qwen3-8B-Base
        & Last to first patching & 0.464 & 1.080 & 0.213 \\
        & First to last patching & 0.306 & 0.218 & 0.386 \\
        & KLPatch & 0.572 & 1.191 & 0.251 \\
        \midrule
        Pythia-6.9b
        & Last to first patching & 0.183 & 0.001 & 0.143 \\
        & First to last patching & 0.309 & 0.001 & 0.173 \\
        & KLPatch & 0.360 & 0.001 & 0.184 \\
        \bottomrule
    \end{tabular}}
    \vspace{10pt}
    \caption{\textbf{Per-task generation AUIC.}  KLPatch has the most consistent generation AUIC across tasks.}
    \label{tab:generation-auic}
\end{table}

\section{Llama-3.2-3B KLPatch Results}\label{apx:llama}
In this section, we discuss the KLPatch of Llama 3.2 3B, which requires rather special considerations to achieve smooth interpolation. 

Figure \ref{fig:klpatch-llama} shows that KLPatch severely underperforms first to last patching, suggesting that KLPatch might not universally improve over naive patching orders. 
However, after enforcing a heuristic constraint that the first student layer (layer 1) be patched first, then the KLPatch algorithm recovers the first to last patching.

One possible reason for this discrepancy could be the initialization of the student model. 
~\citet{kangaslahti2026boomerang} initialize the student model by keeping the first two layers from the teacher model. 
In contrast, when distilling other student models, namely Qwen 4B Base, Qwen 8B Base, and Pythia 6.9b, the student is initialized by keeping the last two layers. 
We also see in Table~\ref{tab:llama-cosine} that Llama-3.2-3B has very low cosine similarity between the outputs of the first two layers compared to the other models. 
This supports model-specific layer-importance differences that we suspect also contribute to the differences in patching order. 
In fact, this difference in student model initialization also led ~\citet{kangaslahti2026boomerang} to adopt the first to last patching strategy rather than the last to first patching strategy to achieve smooth interpolation.
While these results show that Llama models require special treatment of KLPatch to achieve good model size interpolation performance, they also highlight that the effect of patching order is strongly influenced by the student model initialization. 

\begin{table}[h]
    \centering
    \small
    \begin{tabular}{@{}lc@{}}
        \toprule
        \textbf{Model} &
        \textbf{Layer 1 \& 2 Cosine Similarity} \\
        \midrule
        Llama-3.2-3B  & 0.07 \\
        Qwen3-4B-Base & 0.81 \\
        Qwen3-8B-Base & 0.90 \\
        Pythia-6.9b   & 0.96 \\
        \bottomrule
    \end{tabular}
    \vspace{10pt}
    \caption{\textbf{Cosine Similarity Comparison Across Models.} We compute the average cosine similarity between the output activations of layers 1 and 2 of each model on our calibration set. Llama-3.2-3B has much lower cosine similarity between layer 1 and 2 than the other models tested, suggesting model-specific layer importance differences.}   
     \label{tab:llama-cosine}
     \vspace{-10pt}
\end{table}

\begin{figure*}[!ht]
  \centering
  \includegraphics[width=\linewidth]{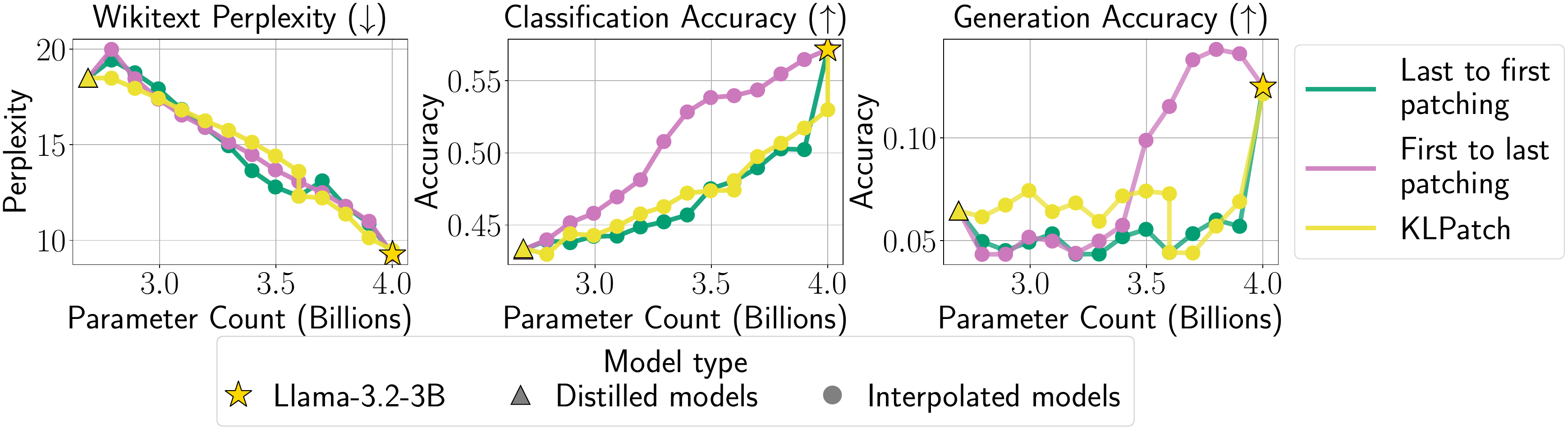}
  \caption{
    \textbf{Standard KLPatch for Llama-3.2-3B.}
    The standard KLPatch algorithm does not recover the optimal patching order.
  }
  \label{fig:klpatch-llama}
\end{figure*}

\begin{figure*}[!ht]
  \centering
  \includegraphics[width=\linewidth]{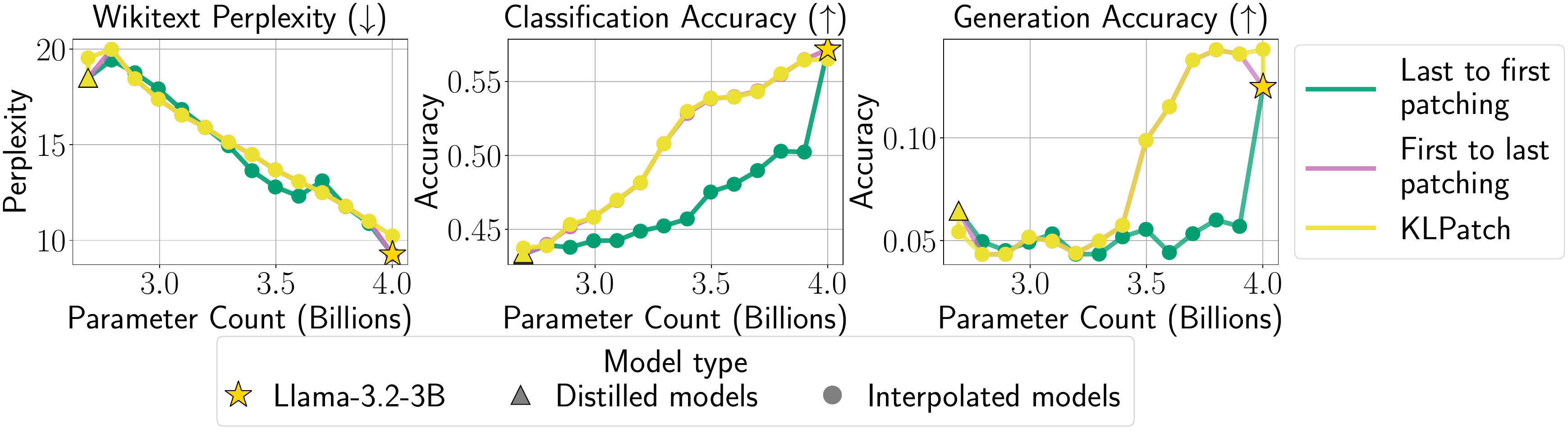}
  \caption{
    \textbf{KLPatch for Llama-3.2-3B by first patching Layer 1.}
    Enforcing layer $1$ as the first patched layer allows KLPatch to recover the optimal patching order and improves performance by avoiding an initially suboptimal patching choice.
  }
  \label{fig:klpatch-llama-fixed}
\end{figure*}

\section{KLPatch Ablations}\label{apx:klpatch-ablations}

We show additional ablations on the KLPatch algorithm, including using metrics other than KL Divergence (\ref{apx:klpatch-metric-ablation}), a non-iterative greedy approach (\ref{apx:klpatch-non-iterative}), additional search algorithms (\ref{apx:klpatch-search-comparison}), and calibration set ablations (\ref{apx:klpatch-calibration-ablation}).

\subsection{KLPatch Performs Best with KLDivergence}\label{apx:klpatch-metric-ablation}

\paragraph{Setup.} We compare the standard KLPatch setup with other metrics used to compute $\ell_i$ in Algorithm \ref{alg:greedy-patching-order} in place of KL divergence. Specifically, we test last-layer cosine similarity between the student and teacher, $\mathcal{L}_{\cos}^{(N)}(x_{<j};\vtheta_\mT,\vtheta_\mS)$, which we denote as cosine patching, as well as perplexity, which we call perplexity patching. We compute both of these metrics on the same validation set used to calculate KL Divergence in the standard KLPatch setup.

\paragraph{Results.} KLPatch produces the strongest interpolation behavior across models and classification and generation tasks (Figures \ref{fig:klpatch-metric-ablation-classification}-\ref{fig:klpatch-metric-ablation-wikitext}). Both KLPatch and perplexity patching signficantly outperform cosine patching. 
Perplexity patching has somewhat similar performance to KLPatch, especially on classification tasks, as for strong teachers, we expect that KL divergence to the teacher provides very similar signal to perplexity.
However, while perplexity patching is a strong baseline, KLPatch provides the most consistent results across models, especially in the generation setting.

\begin{figure*}[!ht]
  \centering
  \vspace{-0.2cm}
  \includegraphics[width=\linewidth]{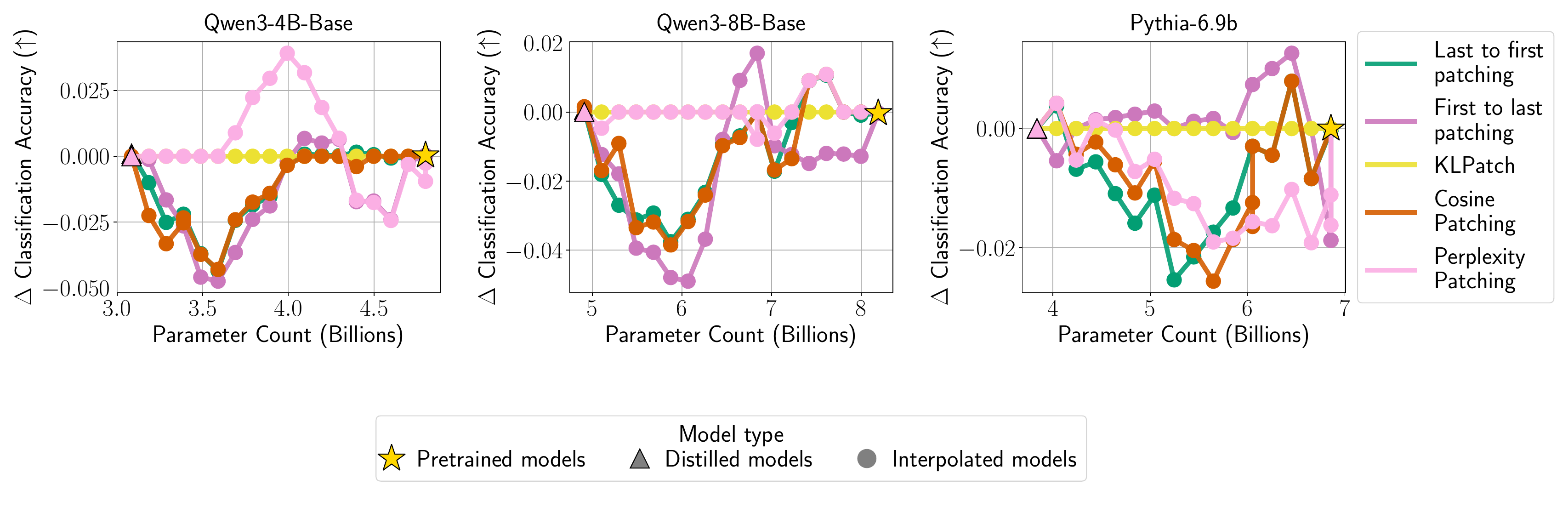}
  \vspace{-0.5cm}
  \caption{
  \textbf{Change in classification accuracy with respect to KLPatch for ablated metrics.} 
  Ablating the metrics used to compute $\ell_i$ in Algorithm \ref{alg:greedy-patching-order} produces interpolation curves worse than or on par with KLPatch on classification tasks.}
  \label{fig:klpatch-metric-ablation-classification}
  \vspace{-0.5cm}
\end{figure*}

\begin{figure*}[!ht]
  \centering
  \includegraphics[width=\linewidth]{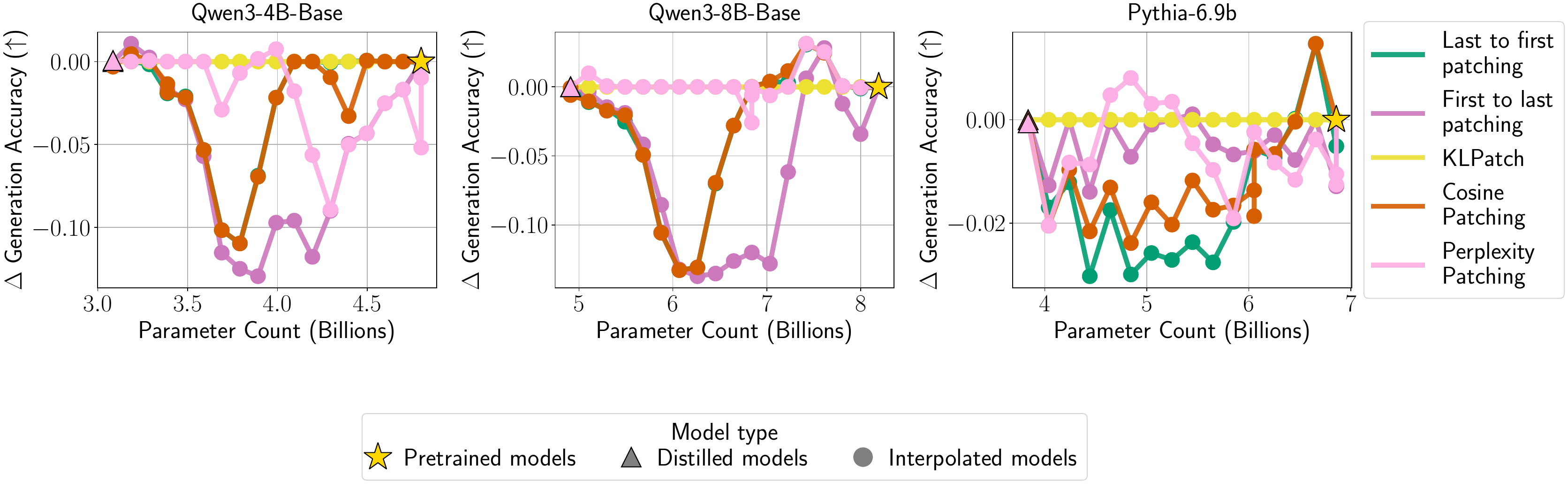}
  \vspace{-0.5cm}
  \caption{
  \textbf{Change in generation accuracy for KLPatch with ablated metrics.}
  For generation tasks, cosine patching and perplexity patching baselines have worse or similar performance to KLPatch across models.}
  \label{fig:klpatch-metric-ablation-generation}
  \vspace{-0.5cm}
\end{figure*}

\begin{figure*}[!ht]
  \centering
  \includegraphics[width=\linewidth]{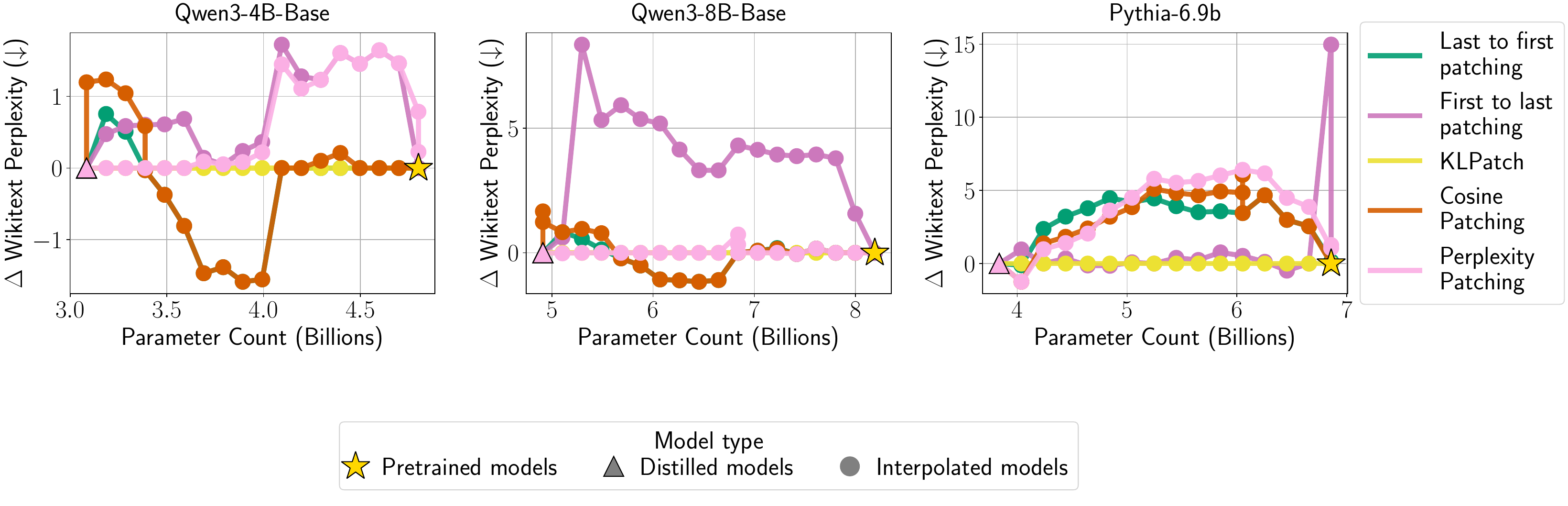}
  \vspace{-0.5cm}
  \caption{
  \textbf{Change in wikitext perplexity with ablated metrics for KLPatch.}
  KLPatch has similar or better wikitext perplexity to other metrics.}
  \label{fig:klpatch-metric-ablation-wikitext}
  \vspace{-0.5cm}
\end{figure*}
\clearpage

\subsection{Iterative KLPatch Outperforms Non-Iterative Approaches}\label{apx:klpatch-non-iterative}

\paragraph{Setup.} We test the standard KLPatch setup against non-iterative baselines. Specifically, we compare against a version of KLPatch, which we call KLInitial, that performs a single iteration of KLPatch and orders all of the layers in terms of their KL $\ell_i$ when only that single layer is patched. We also include baselines that evaluate layer importance without patching any layers. Specifically, we rank student layer importance using the use the block influence metric proposed by \citep{men2024shortgpt}, which calculates the cosine similarity between the inputs and outputs of a given layer, as layers with higher input/output cosine similarity make smaller changes to the residual stream. We then patch layers in order of highest to lowest importance. We also include the logit lens approach \citep{logitlens}, which is a popular method that maps activations to the vocabulary space by multiplying them with the unembedding matrix, as an additional baseline. For each student layer, we use the logit lens technique to compute the KL divergence between the mapped output activations of the student layer and its corresponding teacher block. We then patch student layers in order of highest to lowest KL divergence to their teacher blocks.

\paragraph{Results.} We report the results comparing KLPatch to non-iterative approachs in Figures \ref{fig:klpatch-initial-ablation-classification}-\ref{fig:klpatch-initial-ablation-wikitext}. For classification tasks and perplexity, KLPatch has performance similar to that of KLInitial for smaller interpolated model sizes and outperforms KLInitial for larger interpolated model sizes. For generation tasks, KLPatch outperforms KLInitial across model sizes. 
This indicates that information about the compounding effect of patching intermediate layers (which we obtain from KLPatch but not KLInitial) is especially useful for generation tasks and for intermediate models with more layers patched.
Block influence and logit lens underperform KLPatch and KLInitial, especially for smaller interpolated model sizes, indicating that it is challenging to obtain patching orders with comparable performance to KLPatch from initial importance or alignment information alone.

\begin{figure*}[!ht]
  \centering
  \includegraphics[width=\linewidth]{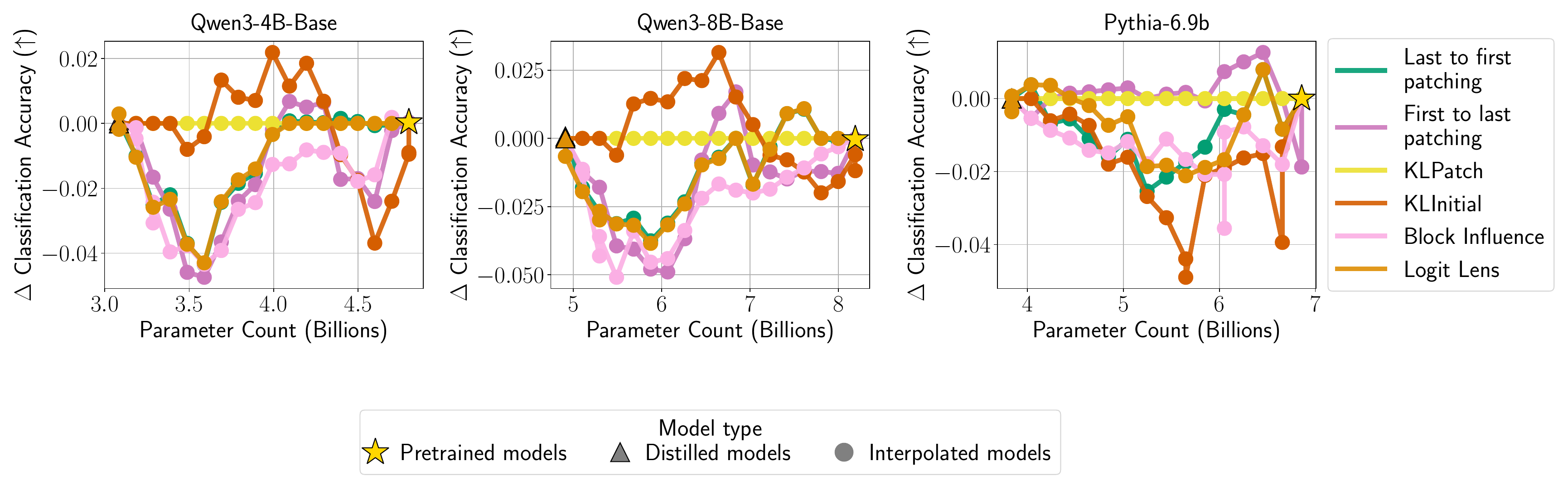}
  \caption{
  \textbf{Difference in classification accuracy between non-iterative approaches and KLPatch.}
  KLPatch has comparable or better performance than non-iterative approaches on classification tasks. }
  \label{fig:klpatch-initial-ablation-classification}
\end{figure*}

\begin{figure*}[!ht]
  \centering
  \includegraphics[width=\linewidth]{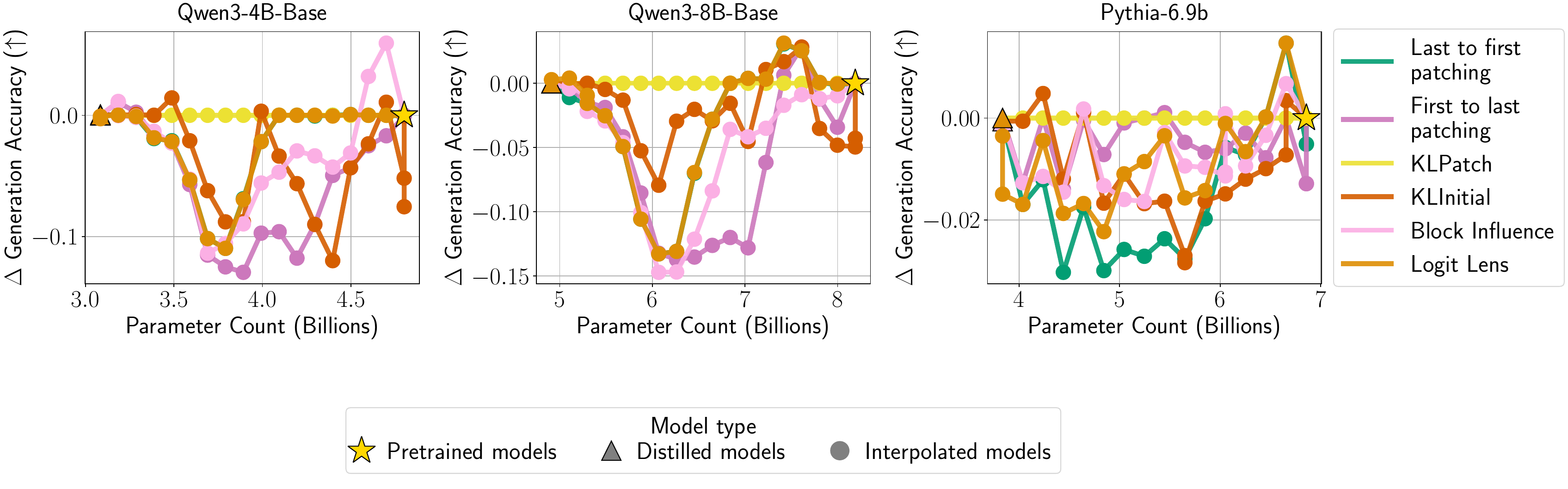}
  \caption{
  \textbf{Difference in generation accuracy between non-iterative approaches and KLPatch.}
  KLPatch has better performance than non-iterative approaches on generation tasks.}
  \label{fig:klpatch-initial-ablation-generation}
\end{figure*}

\begin{figure*}[!ht]
  \centering
  \includegraphics[width=\linewidth]{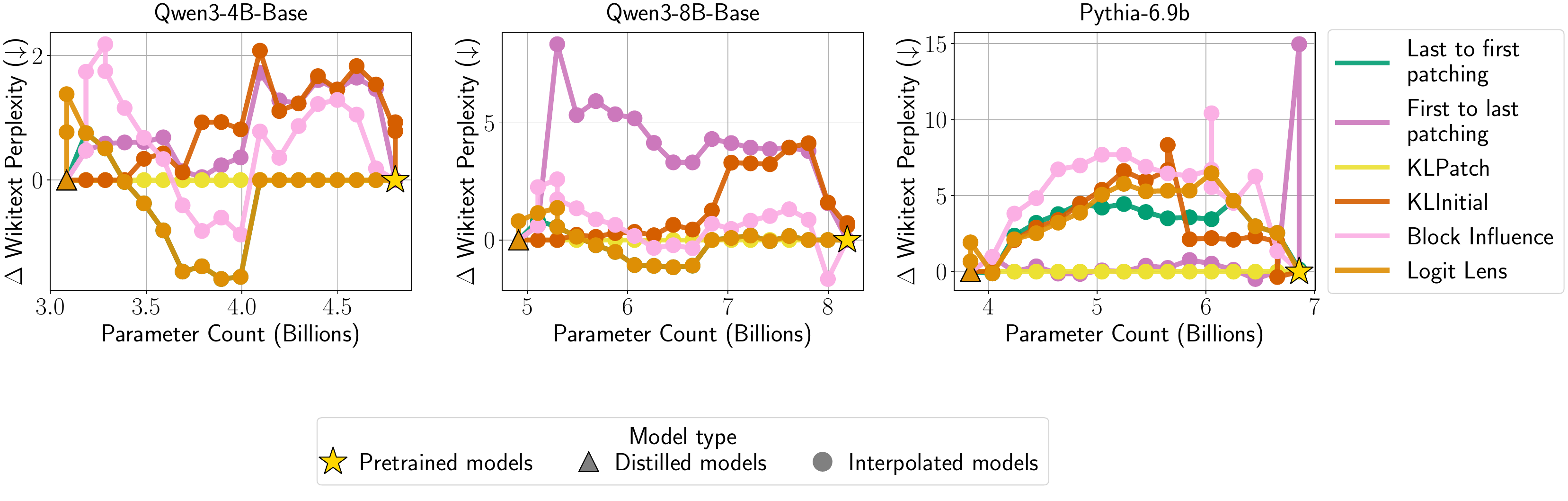}
  \caption{
  \textbf{Difference in wikitext perplexity between non-iterative approaches and KLPatch.}
  KLPatch has similar or better wikitext perplexity to non-iterative approaches.}
  \label{fig:klpatch-initial-ablation-wikitext}
\end{figure*}

\subsection{Additional search algorithms}\label{apx:klpatch-search-comparison}

In Table~\ref{tab:search-method-comparison} below, we compare KLPatch to beam search and 1-step lookahead algorithms. We note that these search algorithms are computationally much more expensive than KLPatch, so they trade search compute for exploring a larger set of strong patching orders. We find that spending additional compute on search results in very small or no gains in performance and in some cases causes overfitting to the calibration set, so KLPatch produces near-best performance with limited compute.

\begin{table}[t]
    \centering
    \resizebox{\linewidth}{!}{%
    \begin{tabular}{llcc}
        \toprule
        \textbf{Model} & \textbf{Metric} &
        \textbf{Classification AUIC ($\uparrow$)} &
        \textbf{Wikitext AUPIC ($\downarrow$)} \\
        \midrule
        Qwen3-4B-Base
        & Beam Search (Beam Width 2) & 0.531 & 0.541 \\
        & Beam Search (Beam Width 3) & 0.534 & 0.538 \\
        & 1-Step Lookahead           & 0.531 & 0.541 \\
        & KLPatch                    & 0.533 & 0.539 \\
        \midrule
        Qwen3-8B-Base
        & Beam Search (Beam Width 2) & 0.505 & 0.784 \\
        & Beam Search (Beam Width 3) & 0.505 & 0.784 \\
        & 1-Step Lookahead           & 0.569 & 0.361 \\
        & KLPatch                    & 0.571 & 0.357 \\
        \midrule
        Pythia-6.9b
        & Beam Search (Beam Width 2) & 0.362 & 0.596 \\
        & Beam Search (Beam Width 3) & 0.386 & 0.577 \\
        & 1-Step Lookahead           & 0.362 & 0.596 \\
        & KLPatch                    & 0.379 & 0.567 \\
        \bottomrule
    \end{tabular}%
    }
    \vspace{10pt}
    \caption{\textbf{Spending more compute on search yields limited gains.} We compare more computationally expensive search algorithms to KLPatch and find limited or no gains in performance, so KLPatch provides a strong option for interpolation while requiring limited compute.}
    \label{tab:search-method-comparison}
\end{table}

\subsection{Ablating the calibration set}\label{apx:klpatch-calibration-ablation}

We test varying calibration set sizes $|\mathcal{D}_{\textrm{cal}}|$ used for evaluating KL divergence in the KLPatch algorithm and report the mean and standard deviation over 5 random seeds in Table \ref{tab:klpatch-calibration-ablation}. We find that KLPatch performance stabilizes around $|\mathcal{D}_{\textrm{cal}}|=64$ across models and downstream tasks, so we use this calibration set size in Section \ref{sec:experiment:klpatch}. 

\begin{table}[!th]
\centering
\begin{tabular}{llccc}
\toprule
Model & $|\mathcal{D}_{\textrm{cal}}|$ & Classification AUIC & Generation AUIC & AUPIC \\
\midrule
\multirow{7}{*}{Qwen3-4B-Base} & 16 & $0.486 \pm 0.066$ & $0.351 \pm 0.112$ & $0.575 \pm 0.042$ \\
 & 32 & $0.534 \pm 0.017$ & $0.428 \pm 0.021$ & $0.553 \pm 0.030$ \\
 & 64 & $0.533 \pm 0.001$ & $0.442 \pm 0.001$ & $0.539 \pm 0.001$ \\
 & 128 & $0.533 \pm 0.001$ & $0.442 \pm 0.002$ & $0.539 \pm 0.001$ \\
 & 256 & $0.533 \pm 0.001$ & $0.442 \pm 0.002$ & $0.539 \pm 0.000$ \\
 & 512 & $0.533 \pm 0.001$ & $0.442 \pm 0.002$ & $0.539 \pm 0.000$ \\
 & 1024 & $0.533 \pm 0.001$ & $0.442 \pm 0.002$ & $0.539 \pm 0.000$ \\
\midrule
\multirow{7}{*}{Qwen3-8B-Base} & 16 & $0.515 \pm 0.079$ & $0.414 \pm 0.108$ & $0.611 \pm 0.333$ \\
 & 32 & $0.553 \pm 0.027$ & $0.471 \pm 0.069$ & $0.449 \pm 0.180$ \\
 & 64 & $0.570 \pm 0.001$ & $0.502 \pm 0.002$ & $0.357 \pm 0.007$ \\
 & 128 & $0.571 \pm 0.002$ & $0.503 \pm 0.001$ & $0.361 \pm 0.000$ \\
 & 256 & $0.571 \pm 0.002$ & $0.503 \pm 0.001$ & $0.361 \pm 0.000$ \\
 & 512 & $0.570 \pm 0.001$ & $0.503 \pm 0.001$ & $0.361 \pm 0.000$ \\
 & 1024 & $0.569 \pm 0.000$ & $0.503 \pm 0.000$ & $0.361 \pm 0.000$ \\
 \midrule
 \multirow{7}{*}{Pythia-6.9B} & 16 & $0.376 \pm 0.011$ & $0.067 \pm 0.002$ & $0.5666 \pm 0.066$ \\
 & 32 & $0.378 \pm  0.007$ & $0.067 \pm 0.003$ & $0.565 \pm 0.023$ \\
 & 64 & $0.379 \pm 0.006$ & $0.067 \pm 0.002$ & $0.567 \pm 0.017$ \\
 & 128 & $0.379 \pm 0.004$ & $0.067 \pm 0.001$ & $0.567 \pm 0.016$ \\
 & 256 & $0.379 \pm 0.001$ & $0.067 \pm 0.001$ & $0.567 \pm 0.001$ \\
 & 512 & $0.379 \pm 0.000$ & $0.067 \pm 0.000$ & $0.567 \pm 0.000$ \\
 & 1024 & $0.379 \pm 0.000$ & $0.067 \pm 0.000$ & $0.567 \pm 0.000$ \\ \\
\bottomrule
\end{tabular}
\vspace{0.2cm}
\caption{\textbf{KLPatch performance summary with varying calibration set size.} Values are reported as mean $\pm$ standard deviation, computed over 5 random seeds.}
\label{tab:klpatch-calibration-ablation}
\end{table}

\end{document}